%% file: main.tex
\documentclass{article}
\usepackage{arxiv}

\usepackage[noorphans,vskip=0ex]{quoting}

\usepackage{microtype}
\usepackage{graphicx}
\usepackage{subfigure}
\usepackage{booktabs} %
\usepackage{natbib}
\usepackage{color}
\usepackage{algorithm}
\usepackage{algorithmic}

\input{math_commands}

\usepackage{bbm}
\usepackage{amssymb}
\usepackage[shortlabels]{enumitem}
\usepackage{listings}
\usepackage{hyperref}
\usepackage{xcolor}
\usepackage{autonum}
\usepackage{tcolorbox}
\usepackage{tikz-cd}
\tcbuselibrary{theorems}

\definecolor{codegreen}{rgb}{0,0.6,0}
\definecolor{codegray}{rgb}{0.5,0.5,0.5}
\definecolor{codepurple}{rgb}{0.58,0,0.82}
\definecolor{backcolour}{rgb}{0.95,0.95,0.92}

\lstdefinestyle{mystyle}{
    backgroundcolor=\color{backcolour},   
    commentstyle=\color{codegreen},
    keywordstyle=\color{magenta},
    numberstyle=\tiny\color{codegray},
    stringstyle=\color{codepurple},
    basicstyle=\ttfamily\footnotesize,
    breakatwhitespace=false,         
    breaklines=true,                 
    captionpos=b,                    
    keepspaces=true,                 
    numbers=left,                    
    numbersep=5pt,                  
    showspaces=false,                
    showstringspaces=false,
    showtabs=false,                  
    tabsize=2
}

\newcommand{\dnameU}{Objective-Driven Dynamics}
\newcommand{\fnameU}{Objective-Driven Dynamical Stochastic Field}
\newcommand{\dname}{objective-driven dynamics}
\newcommand{\fname}{objective-driven dynamical stochastic field}

\title{\fontsize{15pt}{16pt}\selectfont A Framework for Objective-Driven Dynamical Stochastic Fields}

\author{%
Yibo Jacky Zhang\\
Stanford University\\
\texttt{yiboz@stanford.edu}
\And 
Sanmi Koyejo\\
Stanford University\\
\texttt{sanmi@stanford.edu}
}

\begin{document}
\maketitle

\begin{abstract}

    Fields offer a versatile approach for describing complex systems composed of interacting and dynamic components. In particular, some of these dynamical and stochastic systems may exhibit goal-directed behaviors aimed at achieving specific objectives, which we refer to as \textit{intelligent fields}. However, due to their inherent complexity, it remains challenging to develop a formal theoretical description of such systems and to effectively translate these descriptions into practical applications. In this paper, we propose three fundamental principles to establish a theoretical framework for understanding intelligent fields: complete configuration, locality, and purposefulness. Moreover, we explore methodologies for designing such fields from the perspective of artificial intelligence applications. This initial investigation aims to lay the groundwork for future theoretical developments and practical advances in understanding and harnessing the potential of such objective-driven dynamical stochastic fields.
        
\end{abstract}

\input{sections/intro.tex}

\input{sections/overview.tex}

\input{sections/field.tex}

\input{sections/IF.tex}

\input{sections/conclusion.tex}

\bibliographystyle{plainnat}
\bibliography{ref}

\input{sections/appendix.tex}

\end{document}

%% file: math_commands.tex
\usepackage{amsmath,amsfonts,bm, amsthm, mathtools}

\newtheorem{theorem}{Theorem}

\newtheorem{corollary}[theorem]{Corollary}
\newtheorem{principle}{Principle}

\newtheorem{proposition}[theorem]{Proposition}
\newtheorem{fact}[theorem]{Fact}

\newtheorem{lemma}[theorem]{Lemma}
\newtheorem{definition}[theorem]{Definition}

\numberwithin{theorem}{section}
\numberwithin{equation}{section}

\newcommand*\diff{\mathop{}\!\mathrm{d}}

\newcommand{\vspan}[1]{\texttt{span}\left(#1\right)}

\newcommand{\ket}[1]{\lvert #1\rangle}
\newcommand{\bra}[1]{\langle #1 \rvert}
\newcommand{\braket}[2]{\langle #1 \mid #2 \rangle}

\def\field{{\omega}}
\def\gset{{\Omega}}
\def\nbhd{{\gU}}

\def\sfield{{\ket{\varphi}}}
\def\sfcoef{{\varphi}}
\def\avgfield{{\ket{\bar \varphi}}}

\def\lH{{\widetilde{\mathcal{H}}}}
\def\subH{{\Delta\mathcal{H}}}
\def\bG{{{\mathbf G}}}
\def\bH{{{\mathbf H}}}
\def\bM{{{\mathbf M}}}
\def\bN{{{\mathbf N}}}
\def\bS{{{\mathbf S}}}
\def\lG{{\widetilde{\mathbf G}}}
\def\lM{{\widetilde{\mathbf M}}}
\def\lN{{\widetilde{\mathbf N}}}
\def\avgfcoef{{\bar \varphi}}
\def\Dt{{\Delta t}}
\def\dt{{\diff t}}
\def\ds{{\diff s}}

\def\pps{{ \bm{\Gamma}}}
\def\ppss{{ \gamma}}
\def\ppsv{{\bar \gamma}}
\def\ppsp{{{\mathbf{P}}}}
\def\ppsq{{{\mathbf{Q}}}}
\def\ppsadj{{\Lambda}}
\def\genset{{\gP}}
\def\act{{{\widetilde{\mathbf A}}}}
\def\proj{{\mathbf{\Pi}}}
\def\infevo{{ \mathbf{\Phi}}}

\def\eqref#1{equation~\ref{#1}}

\def\1{\bm{1}}

\def\rmF{{\mathbf{F}}}

\DeclareMathAlphabet{\mathsfit}{\encodingdefault}{\sfdefault}{m}{sl}
\SetMathAlphabet{\mathsfit}{bold}{\encodingdefault}{\sfdefault}{bx}{n}

\def\gA{{\mathcal{A}}}
\def\gB{{\mathcal{B}}}

\def\gD{{\mathcal{D}}}

\def\gH{{\mathcal{H}}}

\def\gM{{\mathcal{M}}}
\def\gN{{\mathcal{N}}}
\def\gO{{\mathcal{O}}}
\def\gP{{\mathcal{P}}}

\def\gU{{\mathcal{U}}}

\def\gX{{\mathcal{X}}}

\def\sC{{\mathbb{C}}}

\def\sR{{\mathbb{R}}}

\newcommand{\E}{\mathbb{E}}

%% file: sections/intro.tex
\section{Introduction}\label{sec:intro}

A field refers to a theoretical framework that assigns configurations to every point in spacetime, capturing local interactions. Fields are ubiquitous across a wide range of systems where complex global behaviors can emerge from simple local rules. In particular, we are interested in dynamical and stochastic fields that evolve to achieve specific objectives. We refer to such fields as \textit{intelligent fields}, reflecting their capacity to model objective-driven behavior in systems ranging from artificial intelligence to neural processes.

Developing a fundamental understanding of these objective-driven dynamical stochastic fields poses a significant challenge. It raises several critical questions: What underlying principles give rise to their intricate behaviors? How can we develop a formal mathematical framework to better understand such systems? Could uncovering these principles enable us to design intelligent fields, perhaps as approaches to artificial intelligence? This paper is motivated by these abstract and foundational questions.

Our approach begins with postulating three principles that could underpin such a system. 
The first principle asserts that there is an evolving configuration that completely characterizes the dynamics of the system. The second principle, locality, suggests that the dynamics occur within a spatial-temporal context. Lastly, the principle of purposefulness proposes that the behavior of the system is directed by its objectives. 
Although seemingly simple, these principles lead us to develop a framework for objective-driven dynamical stochastic fields, i.e., the intelligent fields. 
We provide a formal mathematical development of the framework and discuss methodologies for designing the behavior of these fields with a view towards potential applications in artificial intelligence.

%% file: sections/overview.tex
\subsection*{An Overview of the Framework}\label{sec:overview}

Let us first introduce, in a general way, the three principles and the framework for intelligent fields, offering an overview of the paper. We provide abstract definitions of the three principles and discuss the rationale behind them. The concrete definitions are specified in later sections.

The intuition for the first principle relates to the concept of a \textit{complete configuration}. Consider a system described by a time-varying configuration. If the configuration is complete, meaning that it does not require additional information to determine the system, then the current configuration is sufficient to describe the future. We consider stochastic systems under realism, meaning that these configurations evolve probabilistically and exist independently of observation. Consequently, a complete configuration at any given moment must fully determine the future probabilistic evolution of the system without relying on past configurations, making the dynamics Markovian.  In particular, the configuration being complete implies that the transition rule governing the evolution of the configuration can be formulated to be time-invariant, i.e., homogeneous. Otherwise, additional external information would be necessary to define how the transition rules themselves change over time, contradicting the assumption of completeness.
The following statement summarizes the above intuition.

\begin{principle}[\textbf{Complete Configuration}]\label{pinciple:config}
    The system is described by a complete configuration at every time, and the dynamics of the configuration are Markovian and time-invariant.
\end{principle}

It would be quite boring if it were merely a large, plain configuration evolving on its own. Instead, imagine the system to be composed of two entities, each characterized by its own configuration. Each entity’s configuration has two parts: one private and one shared. The private part is accessible only to the entity itself, whereas the shared part is observable by both entities. Thus, at any given moment, each entity observes its own private configuration along with the shared configurations from both entities, but not the other entity's private configuration. As a result, each entity evolves based solely on this partial information.

This is really a notion of \textit{locality}. Imagine we continue this division, and we would have a collection of entities $\gX$ equipped with a topology specifying the neighborhood relations of each $x$ in $\gX$.
Therefore, the system configuration is divided into local configurations associated with each $x\in \gX$, and the time-evolution of a local configuration only depends on the observed configurations of its neighbors. Each point in spacetime is assigned a value satisfying some local relations, and thus we can view the whole system as a field where locality is incorporated. 

\begin{principle}[\textbf{Locality}]\label{pinciple:locality}
    The system is modeled by a field, i.e., each point $x$ in the space $\gX$ is associated with a local entity, and the time-evolution of the local configuration only depends on  its local neighborhood.
\end{principle}

So far, we have described a self-evolving complex system as a dynamical stochastic field. But what governs its evolution? In physical systems, evolution is often determined by the principle of least action, where the system follows a trajectory that minimizes the action, defined as the integral of the system’s Lagrangian. Similarly, a system's stationary state may correspond to its minimum energy configuration. In the context of intelligent agents, behavior is typically directed towards maximizing cumulative reward. Likewise, machine learning models are trained to minimize a loss function. Across these diverse systems, a common theme can be observed: evolution is driven by the minimization or maximization of a value.

Following this intuition, we call this value the \textit{objective value}, as a generalization of the aforementioned concepts, and we demand our field to satisfy the minimization of this objective value.

\begin{principle}[\textbf{Purposefulness}]\label{pinciple:purpose}
    The system evolves to minimize an objective value. Combined with the locality principle, this implies that each entity in the system evolves to minimize its own objective value.
\end{principle}

Together, the three principles define an objective-driven dynamical stochastic field. This field can be conceptualized as follows: each point in space corresponds to an entity that possesses internal configurations and exchanges signals with its neighboring entities. Each entity seeks to minimize its own objective value, which is determined by local interactions. We explore mechanisms in which the local objective value arises from the objective signals generated and propagated by neighboring entities. From the perspective of artificial intelligence, these local mechanisms can be designed in specific ways to guide the evolution of the field toward achieving a desired global goal. Consequently, the system evolves over time to achieve an objective, with its dynamics described as a field in spacetime. Figure~\ref{fig:illustration} provides an illustration of this model.

\begin{figure}
    \centering
    \includegraphics[width=0.7\linewidth]{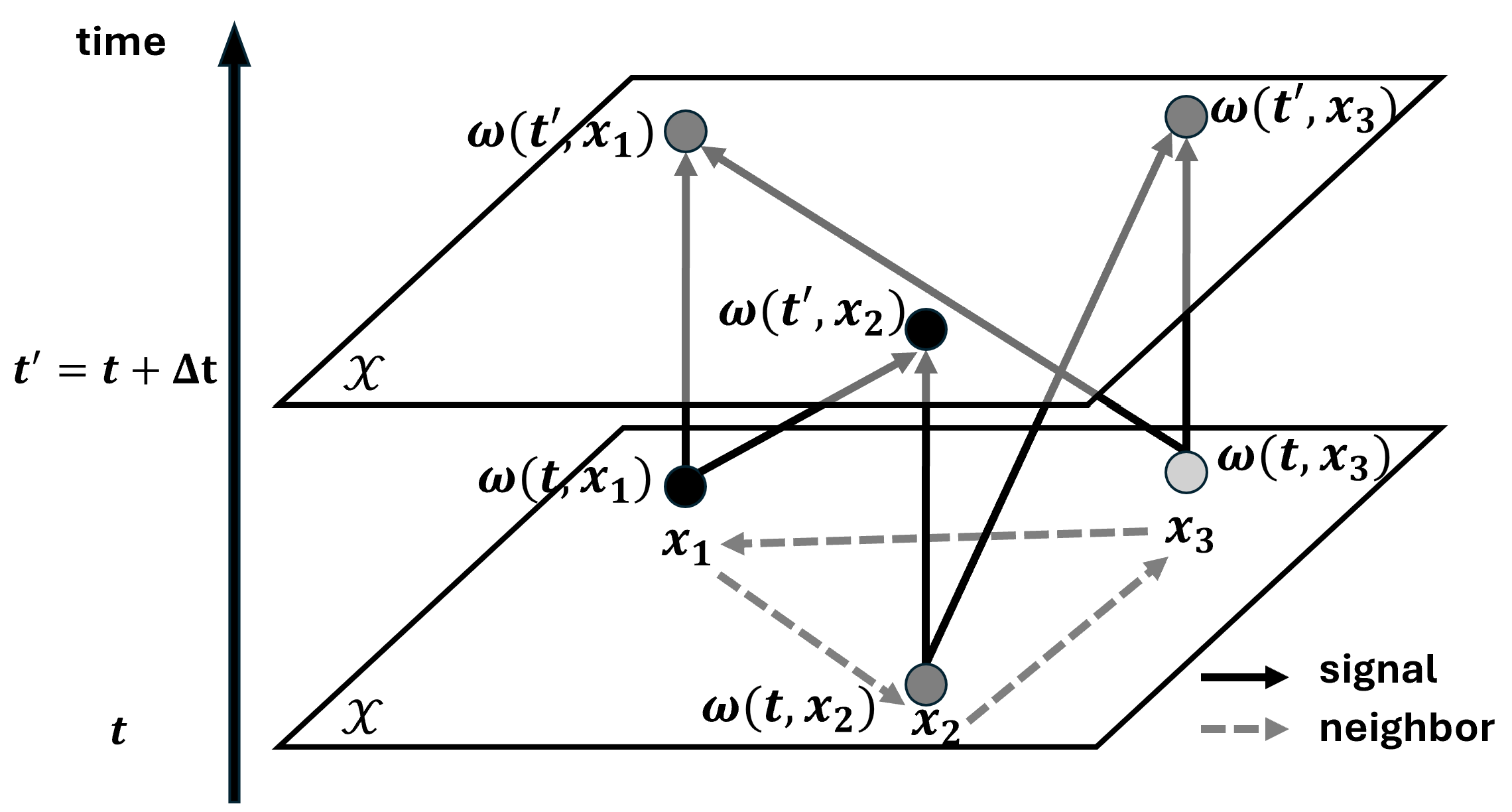}
    \caption{A spacetime diagram illustrating the evolution of local configurations $\field(t, x)$ over time for three entities $x_1, x_2, x_3$ in a discrete space $\gX$. Time progresses vertically from $t$ to $t' = t + \Delta t$, where $\Delta t$ is an infinitesimal time step. Each horizontal layer corresponds to the system at a specific time. Dashed gray arrows represent directed neighboring relationships, indicating directions of signal propagation (e.g., $x_1$ receives signals from $x_3$, but not from $x_2$). Solid arrows represent communication and objective signals, which only propagate forward in time and are limited to immediate neighbors as defined by the dashed links. This shows that the updated local configuration $\field(t', x)$ depends only on the previous configurations of the entity and its neighbors at time $t$. The local objective value is defined as the long-term average of the received objective signals, and each entity evolves to minimize its own local objective.}
    \label{fig:illustration}
\end{figure}

In this paper, we present a formal theoretical framework for analyzing and understanding such a system, which we term an intelligent field, reflecting its capacity to model objective-driven behavior. This framework will become concrete in the following sections as we make specific choices to mathematically define the three principles: complete configuration, locality, and purposefulness.

\paragraph{Structure of the Paper} 
\begin{itemize}
    \item In Section~\ref{sec:field}, we formally define the principles of complete configuration and locality. From the two principles, we build a theoretical model for the dynamical stochastic field. 
    \item Subsequently, in Section~\ref{sec:IF}, we detail the purposefulness principle, complete the dynamical stochastic field with objective-driven behaviors, and explore methodologies for designing the behavior of the system. 
    \item Section~\ref{sec:field} and Section~\ref{sec:IF} form the core technical content of this paper, with key results summarized in Section~\ref{subsec:summary-results}. The proposed framework is general and potentially interdisciplinary; its specific connections to various domains are discussed in Section~\ref{subsec:literature}. 
    \item Finally, we conclude the paper in Section~\ref{sec:conclusion}. For clarity and readability, all proofs are included in the appendix. 
\end{itemize}

%% file: sections/field.tex
\section{Dynamical Stochastic Fields}\label{sec:field}

\setcounter{principle}{0}

Let us build the theoretical tools for the dynamical stochastic fields, derived from the first two principles: complete configuration and locality. To begin, we need to provide a concrete definition of the first principle.
\begin{principle}[\textbf{Complete Configuration}]\label{pinciple:config-detail}
    The system is described by configuration $\field \in \gset$ for a set of configurations $\gset$, and the dynamics of the configuration are continuous-time Markovian and time-invariant. 
\end{principle}
In particular, this work focuses on a finite configuration set $\gset$ to maintain simplicity and clarity of presentation. %
Nevertheless, most of the theoretical development of this work is formulated in a general manner, making it straightforward to extend the theory to infinite configuration sets.

To analyze the system, we find it useful to model its dynamics in the following Hilbert space.%

\textbf{Notation.} Vectors are represented using bra–ket notation, and operators are denoted in bold font. Additional notation will be introduced throughout the paper as needed.

\begin{definition}[\textbf{Hilbert Space $\gH(\gset)$ Constructed from a Set $\gset$}]\label{def:hilbert}
    Consider a set $\gset$, let $\gH(\gset)$ denote a Hilbert space  over the reals $\sR$ constructed by specifying a set of orthonormal basis vectors. The basis vectors and the inner product are defined as: 
	\begin{align}
		\bigl\{\ket{\field}\bigr\}_{\field \in \gset},\qquad \forall \field, \field'\in \gset:\braket{\field}{ \field'} =
		\delta_{\field}^{\field'}, 
	\end{align}
	where we use the Kronecker delta notation, i.e., $\delta_{\field}^{\field'}=1$ only if $\field=\field'$ and $\delta_{\field}^{\field'}=0$ otherwise.
	
	Thus, the Hilbert space is defined to be the closure of the linear span over the reals $\sR$ of the basis vectors:
	\begin{align}
		\gH(\gset)=\vspan{\{\ket{\field } \}_{\field\in \gset}}.
	\end{align}

    When the ground set $\gset$ is clear from context, we simply use $\gH$ to denote $\gH(\gset)$. 
\end{definition}

The dynamics of the system can be described by dynamics within this Hilbert space. Concretely, elements $\sfield\in \gH$ can be used to represent a  probabilistic distribution over $\gset$.

\begin{definition}[\textbf{Normalized Vector}]\label{def:normalized}
    A vector $\sfield=\sum_{\field\in \gset} \sfcoef^{\field} \ket{\field}\in \gH$ is \textit{normalized} if it represents a probability distribution, i.e.,
\begin{align}
    \forall \field: \sfcoef^{\field}\geq 0, \ \quad \text{and}\quad  \sum_{\field\in \gset}  \sfcoef^{\field}=1 .
\end{align}
\end{definition}

Thus, given a normalized vector $\sfield$ that describes the probability distribution of a system at a particular time, the probability of finding the system in configuration $\field$ is 
\begin{align}
    \sfcoef^{\field}=\braket{\field}{\varphi} \in [0, 1].
\end{align}

Next, let us examine the time evolution of the system $\field(t)$, equivalently denoted as $\field_t$. As discussed previously, the time evolution over an interval $\Dt$ can be described by a conditional probability 
\begin{align}
    p_{\Dt}(\field'\mid \field):=p(\field(t+\Dt)=\field' \mid \field(t)=\field). \label{eq:dynamic-1}
\end{align}
We can see that the above quantity converges to the delta function $\delta_{\field}^{\field'}$ as $\Dt\to 0$.  Since the evolution of the system is Markovian, the process can  be characterized by its infinitesimal generator. To be self-contained, we define this operator as follows and show that it governs the evolution of the system. 

In order to define the infinitesimal generator, let us first define the space of linear operators.
\begin{definition}[\textbf{Space of Linear Operators $\gB(\gH(\gset), \gH(\gset'))$}]\label{def:linear-map}
	Let $\gH(\gset), \gH(\gset')$ be two Hilbert spaces constructed by $\gset, \gset'$, respectively. We denote $\gB(\gH(\gset), \gH(\gset'))$ as the space of continuous linear maps from $\gH(\gset)$ to $\gH(\gset')$. Note that a linear operator $\bG\in\gB( \gH(\gset), \gH(\gset'))$ can be represented by 
	\begin{align}
		\bG=\sum_{\field,\field'\in \gset}\ket{\field'} G^{\field'}_{\field} \bra{\field},
	\end{align}
	where $G^{\field'}_{\field}\in \sR$. This means that a linear operator $\bG$ is determined by
	\begin{align}
		G^{\field'}_{\field}=\bra{\field'}\bG\ket{\field}.
	\end{align}
\end{definition}
The infinitesimal generator operator is  such a linear operator  that maps from $\gH$ to itself.  %

\begin{definition}[\textbf{Infinitesimal Generator $\bG$}]\label{def:inf-generator}
	The infinitesimal generator $\bG\colon\gH(\gset)\to \gH(\gset)$ of the system is a linear operator $\bG\in \gB(\gH, \gH)$ determined by the transition rule $p_\Dt$ (\eqref{eq:dynamic-1}) as the following.  For all $\field', \field\in \gset$:
	\begin{align}
		G^{\field'}_{\field}=\bra{\field'}\bG\ket{\field}= \lim_{\Dt\to 0} \frac{ p_{\Dt}(\field'\mid \field)-\delta_{\field}^{\field'}}{\Dt}.
	\end{align}
	By definition, $G^{\field}_{\field}\leq 0$, $G^{\field'}_{\field}\geq 0$ if $\field'\neq \field$,  and $\sum_{\field'}G^{\field'}_{\field}=0$.
\end{definition}

The above infinitesimal generator encapsulates all the information we need to characterize the dynamics of $\ket{\varphi(t)}$, which describes the distribution of system configurations at time $t$. It provides a highly compact formulation of the system dynamics, as shown below. It is called the master equation, or a reformulation of the Kolmogorov forward equations~\citep{kolmogoroff1931analytischen, feller1940integro}.
\begin{proposition}[\textbf{Dynamics of the System}]\label{prop:dynamics-generator}
		The time evolution of the system satisfies the following first-order differential equation:
	\begin{align}
		\frac{\diff }{\dt}\ket{\varphi(t)}=\bG \ket{\varphi(t)}.
	\end{align}
	Its solution is 
	\begin{align}
		\ket{\varphi(t)}=e^{\bG t}\sfield(0).
	\end{align}
	Note that the exponential of the linear operator is defined as $
	e^{\bG t}=1+\sum_{n=1}^{\infty} \frac{1}{n!}(\bG t)^n$. 	When the context is clear,  we  may simply use $\sfield$ to denote $\ket{\varphi(t)}$ for notational ease.
\end{proposition}

Therefore, we have two formulations of the system dynamics: the dynamics of the configuration $\field(t)\in \gset$ itself and the dynamics of the probability distribution $\ket{\varphi(t)}\in \gH$. To distinguish the two formulations, we adopt  the following convention.

\textbf{Notation.} We refer to $\field(t)\in \gset$ as a \textit{\textbf{configuration}} of the system dynamics and a normalized vector $\ket{\varphi(t)}\in \gH$ as a \textit{\textbf{state}} of the system dynamics. A state represents a probability distribution of configurations.

In this paper, we primarily focus on the state perspective, i.e., the dynamics of the probability distribution. We will later demonstrate the equivalence of both perspectives via a path integral formulation at the end of this section. 
Before formally introducing the locality principle in Section~\ref{subsec:field-formulation}, let us warm up by examining the following simplified case.

\subsection{A Two-entity View}\label{subsec:two-entity-field}
As a first step towards the field formulation, consider a system configuration composed of two entities, $x$ and $y$. One might also interpret  this as a partition of the underlying space $\gX$.%

Entity $x$ possesses a configuration $\mu\in \gM$ and $\alpha \in \gA$, where only $\alpha$ is observable to entity $y$. Symmetrically, entity $y$ possesses a configuration $\nu\in \gN$ and $\beta \in \gB$, where only $\beta$ is observable to entity $x$. Therefore, the system configuration is described by:
\begin{align}
    \field=(\alpha, \beta, \mu, \nu)\in \gset,\qquad \text{where}\ \ \gset=\gA\times\gB\times\gM\times\gN.
\end{align}

The locality principle implies that, at time $t$, entity $x$’s immediate next behavior $(\alpha_{t+\Dt}, \mu_{t+\Dt})$ depends only on the currently observed local configurations $(\alpha_t, \beta_t, \mu_t)$, and entity $y$’s next behavior $(\beta_{t+\Dt}, \nu_{t+\Dt})$ depends only on the currently observed local configurations $(\alpha_t, \beta_t, \nu_t)$. 

Formally, the locality principle implies the immediate independence relation, as shown below. Following the notation as stated in \eqref{eq:dynamic-1}, we have
\begin{align}
    p_{\Dt}(\alpha'\beta'\mu'\nu'\mid \alpha\beta\mu\nu) =p_{\Dt}(\alpha'\mu'\mid \alpha\beta\mu)\cdot p_{\Dt}(\beta'\nu'\mid \alpha\beta\nu)+o(\Dt),
    \label{eq:3-1-1}
\end{align}
where $o(\Dt)$ satisfies $\lim_{\Dt\to 0 }\frac{o(\Dt)}{\Dt}=0$, and
\begin{align}
    \forall \nu: p_{\Dt}(\alpha'\mu'\mid \alpha\beta\mu)&=p_{\Dt}(\alpha'\mu'\mid \alpha\beta\mu\nu)+o(\Dt),\\
    \forall \mu: p_{\Dt}(\beta'\nu'\mid \alpha\beta\nu)&=p_{\Dt}(\beta'\nu'\mid \alpha\beta\mu\nu)+o(\Dt).
\end{align}

Therefore, $p_{\Dt}(\alpha'\mu'\mid \alpha\beta\mu)$ characterizes the behavior of entity $x$, while $p_{\Dt}(\beta'\nu'\mid \alpha\beta\nu)$ characterizes the behavior of entity $y$. Together, they define the dynamics of the entire system. 

It is thus interesting to investigate whether we can similarly formulate infinitesimal generators for each of the two entities. To achieve this,  we follow Definition~\ref{def:hilbert} to construct the corresponding Hilbert spaces and Definition~\ref{def:inf-generator} to formulate the infinitesimal generators.

\begin{definition}[\textbf{Infinitesimal Generators $\bM, \bN$ for Two Entities}]\label{def:generator-entity}
	The infinitesimal generator $\bM\colon\gH(\gA\times \gB\times \gM)\to \gH(\gA\times \gM)$ of  entity $x$ is a linear operator determined by 
	\begin{align}
		\bra{\alpha' \mu' }\bM \ket{\alpha \beta \mu }= \lim_{\Dt\to 0} \frac{ p_{ \Dt}(\alpha' \mu' \mid \alpha \beta \mu)-\delta_{\alpha  \mu }^{\alpha' \mu' }}{\Dt}.
	\end{align}
	Similarly, the infinitesimal generator $\bN\colon\gH(\gA\times \gB\times \gN)\to \gH(\gB\times \gN)$ of  entity $y$ is a linear operator determined by
	\begin{align}
		\bra{\beta' \nu' }\bN \ket{\alpha \beta \nu }= \lim_{\Dt\to 0} \frac{ p_{ \Dt}(\beta' \nu' \mid \alpha \beta \nu)-\delta_{\beta\nu}^{\beta'\nu'}}{\Dt}.
	\end{align}
\end{definition}

As we can see, the operator $\bM$ encodes all the information about the transition rules of entity $x$, and $\bN$ encodes the transition rules of entity $y$. Therefore, one should expect that the generator of the whole system, $\bG$, can be expressed in terms of $\bM$ and $\bN$. In order to do this, we need to embed both $\bM, \bN$ into the space $\gB(\gH, \gH)$, the space of continuous linear operators mapping from $\gH(\gset)$ to itself, where $\bG$ resides. The embedding needs to be done such that we can obtain $\bG$ from $\bM$ and $\bN$. Luckily, as we shall see, such embedding arises straightforwardly once we introduce the necessary technical machinery.

\begin{fact}\label{fact:iso-tensor}
    Given two sets $\gset,\gset'$, the Hilbert space $\gH(\gset\times \gset')$ and the tensor product $\gH(\gset )\otimes \gH(\gset')$ are isometrically isomorphic:
    \begin{align}
        \gH(\gset\times \gset') \cong \gH(\gset )\otimes \gH(\gset'),
    \end{align}
    i.e., there exists a bijective linear map between the two spaces that preserves the inner product. This can be intuitively viewed as an equivalence between the two spaces.
\end{fact}

Therefore, we may embed $\gH(\gA\times \gM)$ in $\gH(\gset)=\gH(\gA\times\gB\times\gM\times\gN)$ by using the tensor product  $\gH(\gB )\otimes \gH(\gN)$. 

\textbf{Notation.} We adopt the convention that all tensor products are automatically ordered, e.g.,  $\gH(\gA )\otimes \gH(\gM) \otimes \gH(\gB )\otimes \gH(\gN)$ is equivalent to $\gH(\gA )\otimes \gH(\gB) \otimes \gH(\gM )\otimes \gH(\gN)$. Therefore, the tensor product of basis vectors $\ket{\alpha\mu}\in \gH(\gA\times \gM)$ and $\ket{\beta\nu}\in \gH(\gB\times \gN)$ is $\ket{\alpha\mu}\otimes \ket{\beta\nu}=\ket{\alpha\beta\mu\nu}\in \gH$. Formally, this means we work in a quotient space obtained by identifying tensor products differing only by ordering. For ease of notation, however, we omit explicit references to this quotient operation and simply adopt this ordering as a standard convention throughout this paper. 

Thus, the embeddings are stated in the following definition. 
\begin{definition}[\textbf{Embed $\bM, \bN$ in $\gB(\gH, \gH)$}]\label{def:generator-embed}
	Given  infinitesimal generator $\bM\colon\gH(\gA\times \gB\times \gM)\to \gH(\gA\times \gM)$ of the entity $x$, and  $\bN\colon\gH(\gA\times \gB\times \gN)\to \gH(\gB\times \gN)$ of the entity $y$, we define their embeddings $\bar \bM, \bar \bN$ in the space $\gB(\gH, \gH)$  as the following.  Given  basis vectors $ \ket{\alpha\beta\mu\nu}\in \gH$,
	\begin{align}
		\bar \bM\ket{\alpha\beta\mu\nu}&:= (\bM\ket{\alpha\beta\mu}) \otimes \ket{\beta\nu},\\
		\bar \bN\ket{\alpha\beta\mu\nu}&:= (\bN\ket{\alpha\beta\nu}) \otimes \ket{\alpha\mu}.
	\end{align}
    When it is clear in the context, we simply use $\bM, \bN$ to denote their embeddings $\bar \bM, \bar \bN$, respectively.
\end{definition}
As we can see, the embedding $\bM\in \gB(\gH, \gH)$ is a generator of the whole system where only the entity $x$ is active; and $ \bN\in  \gB(\gH, \gH)$ is a generator of the whole system where only the entity $y$ is active.

Indeed, the generator of the whole system $\bG$, can be expressed in terms of the two generators $ \bM$ and $\bN$ of the individual entities. Moreover,  as we shall show next, the relationship between $\bG$ and $ \bM, \bN$ exhibits a particularly elegant structure. 

\begin{proposition}[\textbf{Decomposition of the Infinitesimal Generator $\bG$}]\label{prop:decomposition}
	The locality principle in the two-entity view (\eqref{eq:3-1-1}) implies
	\begin{align}
		\bG= \bM +  \bN.
	\end{align}
\end{proposition}

We have so far considered the case of a system consisting of two entities. Next, we extend this theoretical model to multiple interacting entities that collectively form a field.

\subsection{Field Formulation}\label{subsec:field-formulation}

Imagine a system composed of multiple entities, each labeled by $x\in \gX$. The locality principle implies that each entity interacts only with its neighbors. In the abstract formulation of our model, as introduced in Section~\ref{sec:intro}, the ground space $\gX$ is only required to have a topology that defines neighborhood relations. In this paper, we specifically model this space as a finite directed graph, leaving more general cases for future study. 

\textbf{Notation.} Following the previous section, each entity has its own configuration $\mu(x)\in \gM(x)$, it sends $\alpha(x, x')\in \gA(x, x')$ to $x'$, and it receives $\alpha(x'', x)\in \gA(x'', x)$ from $x''$. Thus, the space $\gX$ can be modeled by a directed graph. We denote $x\to x'$ if there is an edge going from $x$ to $x'$; and $x\sim x'$ if either $x\to x'$ or $x'\to x$. 

\begin{definition}[\textbf{The Configuration Space of a Field}]\label{def:space}
	In the field formulation, we consider $\gX$ being modeled by a directed graph.  The set system configuration is the product of local configurations, 
	\begin{align}
		\gset=\prod_{x\in \gX} \gset_x,
	\end{align}
	where each local configuration $\gset_x$ is what the local entity $x$ can act on, i.e.,
	\begin{align}
		\gset_x=\gM(x)\times \prod_{x': x\to x'}\gA(x, x').
	\end{align}
	Moreover, we denote the set of configurations observable by entity $x$ as follows:
	\begin{align}
		\nbhd_x=\gM(x)\times \prod_{x': x\to x'}\gA(x, x')\times \prod_{x': x'\to x}\gA(x', x).
	\end{align}
	Therefore, given a system configuration $\field\in \gset$, there are canonical  projections that map it to $\gset_x$ and $\nbhd_x$ via restrictions. We denote these projections as
	\begin{align}
		\field_{|\gset_x}\in \gset_x, \qquad \field_{|\nbhd_x}\in \nbhd_x.
	\end{align}
	In particular, when considering the system dynamics $\field_t$, we denote
	\begin{align}
		\field(t, x):=\field_{t|\gset_x}.
	\end{align}
\end{definition}

Therefore, the locality principle implies that the local dynamics of $\field(t, x)$ only depend on its immediate neighborhood $\nbhd_x$. Formally, it implies an independence relation generalized from \eqref{eq:3-1-1}.

\begin{principle}[\textbf{Locality}]\label{pinciple:locality-detail}
 	The time-evolution of a local configuration $\field(t, x)$ only depends on its local neighborhood $\nbhd_x$. Formally, 
	\begin{align}
		p_\Dt(\field'\mid\field)=\prod_x p_\Dt(\field'_{|\gset_x}\mid\field_{|\nbhd_x}) + o(\Dt),
	\end{align}
	where $p_\Dt(\field'_{|\gset_x}\mid\field_{|\nbhd_x})=p_\Dt(\field'_{|\gset_x}\mid\field)+o(\Dt)$.
\end{principle}

Similarly to the two-entity formulation, the local dynamics of $\field(t, x)$ can be modeled by their local infinitesimal generators.

\begin{definition}[\textbf{Local Infinitesimal Generators $\bG(x)$}]\label{def:generator-field}
	The infinitesimal generator $\bG(x)\colon\gH(\nbhd_x)\to \gH(\gset_x)$ of  entity $x$ is a linear operator determined by 
	\begin{align}
		\bra{\field'_{|\gset_x} } \bG(x) \ket{\field_{|\nbhd_x} }= \lim_{\Dt\to 0} \frac{ p_{ \Dt}(\field'_{|\gset_x} \mid \field_{|\nbhd_x} )-\delta_{\field_{|\gset_x}  }^{\field'_{|\gset_x}}}{\Dt}.
	\end{align}
	We can embed $\bG(x)$ in the larger space $\gB(\gH, \gH)$, the same space in which the system's generator $\bG\colon\gH(\gset)\to \gH(\gset)$ resides. The embedding $\bar \bG(x)$ is determined by
	\begin{align}
		\bar \bG(x) \ket{\field}= (\bG(x)\ket{\field_{|\nbhd_x}})\otimes \ket{\field_{|\gset/\gset_x}},
	\end{align}
	where $\gset/\gset_x=\prod_{x':x'\neq x}\gset_{x'}$ is taking the quotient. For simplicity, we also use $\bG(x)$ to denote its embedding $\bar{\bG}(x)$, as the context will make it clear which version is being referred to.
\end{definition}

One could expect, similar to Proposition~\ref{prop:decomposition}, that the system's generator $\bG$ can be expressed in terms of the local generators $\bG(x)$ in an elegant manner. This is indeed the case, as shown in the following theorem. 

\begin{theorem}[\textbf{Decomposition of the Infinitesimal Generator $\bG$}]\label{thm:decomposition-field}
	Locality (Principle~\ref{pinciple:locality-detail}) implies
	\begin{align}
		\bG= \sum_x \bG(x).
	\end{align}
\end{theorem}

In fact, there is another consequence we may expect from the principle of locality, i.e., the commutation relations of the local generators. Intuitively, $\bG(x)$ represents an immediate action at point $x$. Thus, if another point $x'$ is distant from $x$, then the order of applying $\bG(x)$ and $\bG(x')$ should not matter. Indeed, in such cases, these operators commute.
\begin{proposition}[\textbf{Commutation Relations of the Local Generators}]\label{prop:commute}
	If $x, x'$ are not neighbors, i.e., $x \nsim x'$, their generators commute. Formally,  
	\begin{align}
		[\bG(x), \bG(x')]=0, \qquad \text{if } x\nsim x',
	\end{align}
	where the commutator $[\bG(x), \bG(x')]=\bG(x)\bG(x')-\bG(x')\bG(x)$.
\end{proposition}

\subsection{A Path Integral Formalism}\label{subsec:path-integral}
The infinitesimal generator formalism describes how the system state, i.e., the probability distribution over configurations, evolves over time. Alternatively, one may take a perspective that directly traces  the trajectory of the evolving configuration itself, which is characterized by the path integral formalism. As we will show, this perspective offers some new insights.

We begin by considering the entire system and examining the probability that it starts in configuration $\field_0=\field$ at time $0$ and ends up in configuration $\field_T=\field'$ at time $T$. This propagation probability is given by:
\begin{align}
	p_T(\field'\mid\field)= \bra{\field'} e^{\bG T} \ket{\field}. 
\end{align}
We can see that the system may go from $\field$ to $\field'$ via multiple possible paths, each with a different probability. Therefore, one may be curious to observe what happens in the middle and what path the system takes. To do this, let us try to observe the system every $\Dt=T/N$ time for $N$ times. Suppose we observe a path
\begin{align}
	\field_{0:T, \Dt}:=(\field_0, \field_{\Dt}, \field_{2\Dt},\dots,\field_{(N-1)\Dt}, \field_T).
\end{align}
The probability of observing such a path, given that the system starts at $\field_0$, is
\begin{align}
	p(\field_{0:T, \Dt}\mid\field_0)=\prod_{n=1}^N p_\Dt(\field_{n\Dt}\mid\field_{(n-1)\Dt})=\prod_{n=1}^N \bra{\field_{n\Dt}} e^{\bG\Dt} \ket{\field_{(n-1)\Dt}}
\end{align}
Therefore, the probability $p_T(\field'\mid\field)$ can be obtained by summing over all possible paths that start at $\field_0=\field$ and end at $\field_T=\field'$, i.e.,
\begin{align}
	p_T(\field'\mid\field)=\sum_{\field_{0:T, \Dt}}^{\field\to \field'} p(\field_{0:T, \Dt}\mid\field_0)= \sum_{\field_{0:T, \Dt}}^{\field\to \field'}\prod_{n=1}^N \bra{\field_{n\Dt}} e^{\bG\Dt} \ket{\field_{(n-1)\Dt}},
\end{align}
where $\sum_{\field_{0:T, \Dt}}^{\field\to \field'}$ denotes summation over all paths with $\Dt$ intervals that start at $\field_0=\field$ and end at $\field_T=\field'$. 

Observe that, in the above equation, the left-hand side has nothing to do with the time interval $\Dt$. Thus, it would remain the same as we take the limit $\Dt\to 0$ (equivalently, $N\to \infty$) on the right-hand side.  

It turns out that, with some algebraic manipulations, we can obtain something analogous to the Lagrangian of a physical system. However, unlike systems with a smooth configuration space, our system evolves in a discrete configuration set $\gset$. To properly analyze its jump behavior, we require a specialized function, the unit impulse function, which is a formulation of the Dirac delta function. This should not be confused with the delta indicator $\delta_\field^{\field'}$ that we have been using. 

\begin{definition}[\textbf{Unit Impulse}]\label{def:unit-impulse}
    Given a path of configurations $\field_t\in \gset$, we denote the corresponding unit impulse function, as a reformulation of the Dirac delta function:
\begin{align}
	\delta({\field_t^+\neq \field_t})=
	\begin{cases}
		\infty,\quad &\text{if } \field_t^+\neq \field_t\\
		0, & \text{otherwise},
	\end{cases}
\end{align}
where we use $\field^+_t$ to denote the next immediate configuration. Given a path having jumps occur at time $t_1,t_2,\dots, t_k$, for any function $f\colon\sR\to \sR$, we have
\begin{align}
    \int \dt \ \delta({\field_t^+\neq \field_t})\cdot f(t)=\sum_{i} f(t_i).
\end{align}
\end{definition}

As we will demonstrate shortly, we can identify a term that closely resembles the Lagrangian of a physical system. More specifically, this term appears precisely in the position where the Lagrangian typically appears in the path integral formalism of a physical system.

\begin{definition}[\textbf{The Lagrangian of the System}]\label{def:lagrangian}
    Given the infinitesimal generator $\bG$ of the system, we define
	\begin{align}
		L(w_t, w_t^+):=-G^{\field_t}_{\field_t}-\delta({\field_t^+\neq \field_t})\cdot \log  G^{\field_{t}^+}_{\field_{t}} .
	\end{align}	
\end{definition}
We may also recognize this term as a modified log-likelihood, since $G^{\field_{t}^+}_{\field_{t}}$ represents the probability of the system transitioning  from $\field_{t}$ to $\field_{t}^+$.

By restricting the involved quantities to local neighborhoods, we obtain the field Lagrangian.
\begin{definition}[\textbf{Field Lagrangian}]\label{def:lagrangian-field}
	Given the infinitesimal generator $\bG=\sum_x \bG(x)$ of the system, the field Lagrangian is 
	\begin{align}
		L(w_{t, x}, w_{t, x}^+) := -G(x)^{{\field_t}_{|\gset_x}}_{{\field_t}_{|\nbhd_x}}- \delta({\field_t^+}_{|\gset_x}\neq {\field_t}_{|\gset_x})\cdot \log \left( G(x)^{{\field_t^+}_{|\gset_x}}_{{\field_t}_{|\nbhd_x}}\right).
	\end{align}
\end{definition}
The field Lagrangian is a natural restriction of the system Lagrangian to the local neighborhood, and their relationship is analogous to the decomposition of the generators. 
\begin{proposition}[\textbf{Decomposition of the Lagrangian}]\label{prop:decomp-Lagrangian}
\begin{align}
    L(w_t, w_t^+)=\sum_x L(w_{t, x}, w_{t, x}^+).
\end{align}    
\end{proposition}

Next, to derive the path integral formulation of the system, we first need to figure out how to perform this integration over all paths. 

\textbf{Notation.} We begin by classifying different paths by the number of jumps, i.e., the number of configuration changes. Denote $\field_{0:T, \Dt, k}$ as a path having  $k$ jumps, and define $t_n=n\Dt$. For any function of a path $f(\field_{0:T})$, adopting the conventional path integral notation, we denote the integration of $f$ over all paths going from $\field$ to $\field'$ in time $T$ as follows. 
\begin{align}
    \int \gD \field \ f(\field_{0:T}) :=  \lim_{\Dt\to 0} \sum_{k=0}^N  \sum_{\field_{0:T, \Dt, k}}^{\field\to \field'} \cdot (\Dt)^k  \cdot f(\field_{0:T, \Dt, k}).
\end{align}
Note that the limit $\Dt\to 0$ also implies $N\to \infty$. In this limit,  $(\Dt)^k$ appears as a path measure to ensure convergence by assigning smaller weights to paths with more jumps. 

Having developed the necessary machinery, a path integral formulation of the system can be derived as follows.

\begin{theorem}[\textbf{A Path Integral Formalism of the System}]\label{thm:path-int}
    The probability that the system evolves from $\field$ to $\field'$ after time $T$ can be expressed in the following ways.
    \begin{align}
	p_T(\field'\mid\field)=\bra{\field'}e^{\bG T} \ket{\field}  =\int \gD \field \ \exp\left\{-\sum_x\int_{0}^{T} \dt \ L(w_{t, x}, w_{t, x}^+)\right\}, %
\end{align}
where the path integral is done over paths starting from $\field$ to $\field'$ in time $T$.
\end{theorem}

The path integral formalism of the dynamical stochastic field reveals that, beyond the spatial neighborhood, the temporal neighborhood, i.e., knowing both $\field$ and $\field^+$, is also important in characterizing the field dynamics--as one would expect. This temporal aspect is neglected in our previous formulation based on generators. The following subsection presents a more complete framework that incorporates this temporal characterization.

\subsection{Completing the Generator Formalism and the Deterministic Limit}\label{subsec:complete}

To introduce the temporal characterization to the generator formalism, let us begin by motivating it through the following interesting problem. Suppose the system has a fixed ``energy'' such that the frequency of its jumps between configurations is constant. Formally, noting that $|G^{\field}_{\field}|$ is the rate at which the system jumps from $\field$ to other configurations, let us demand
\begin{align}
	\forall \field:\quad |G^{\field}_{\field}|= \sum_{\field':\field'\neq \field} G^{\field'}_{\field}= K. \label{eq:3-4-4}
\end{align}
Recall that the integral of the Lagrangian essentially indicates the likelihood of a particular path being taken. One may be curious about the expectation of the integral, i.e.,
\begin{align}
    \E\left[\int_0^T \dt\ L(\field_t, \field_t^+) \Bigm| \field_0=\field\right], \label{eq:3-4-1}
\end{align}
where in this case
\begin{align}
    L(w_t, w_t^+)=K-\delta({\field_t^+\neq \field_t})\cdot \log  G^{\field_{t}^+}_{\field_{t}}.
\end{align}
This does not seem easy based on what we have so far. After all, $L(t, w, w^+)$ involves ``impulses'' that make the integral non-standard. We need to introduce certain tools to effectively obtain the above terms. Let us first present some key definitions and then show how they are useful.

Observe that an impulse essentially contributes a value each time a jump occurs. However, this behavior cannot be captured directly by our space $\gH$. Therefore, we need to expand it to include the temporal neighborhood into consideration, resulting in a larger space $\lH(\gset)$.
\begin{definition}[\textbf{ Space $\lH(\gset)$}]\label{def:hilbert-2}
    Given the set of configurations $\gset$, we define 
    \begin{align}
        \lH(\gset):=\gH(\gset\times \gset) \cong \gH(\gset)\otimes \gH(\gset).
    \end{align}
    A basis set of $\lH(\gset)$ is denoted as
    \begin{align}
        \ket{\field'\field}\quad \text{for all }\  \field', \field\in \gset.
    \end{align}
    When the context is clear, we only use $\lH$ to denote $\lH(\gset)$.
\end{definition}
Then, in order to investigate \eqref{eq:3-4-1}, we need to lift the generator $\bG\colon\gH\to \gH$ to $\lG\colon\gH\to \lH$, as follows.

\textbf{Notation.} We denote $\act_{\field'}^{\field}\colon\gH\to \lH$ as a linear operator that transforms $\ket{\field}$ to $\ket{\field'\field}$. Formally
\begin{align}
	\act_{\field'}^{\field}\ket{\field''}=
	\begin{cases}
		\ket{\field'\field},\qquad &\text{if}\ \field=\field''\\
		0, &\text{otherwise}.
	\end{cases} \label{def:action-operator}
\end{align}
\begin{definition}[\textbf{ Lift of the Generator $\lG\colon\gH\to \lH$}]\label{def:lift-generator}
    The lift of generator $\lG\colon\gH\to \lH$ is defined to be
    \begin{align}
        \lG \ket\field=\sum_{\field'\in\gset}G^{\field'}_{\field}\ket{\field'\field},\qquad \text{or equivalently}\qquad \lG=\sum_{\field,\field'\in\gset}G^{\field'}_{\field}\act_{\field'}^{\field}.
    \end{align}
\end{definition}

With the newly introduced constructions, we are now well-positioned  to reformulate \eqref{eq:3-4-1} back to our generator formalism, getting rid of the expectation operator.  Next, we first prove a more general statement and then apply it to \eqref{eq:3-4-1}.
\begin{proposition}\label{prop:lift-operators-a-case}
	Consider a function in the form of 
	\begin{align}
		\gamma(t, \field, \field^+)=\delta({\field_t^+\neq \field_t})\cdot \gamma_{\field_t^+\field_t},
	\end{align}
	It corresponds to an operator $\pps:\lH\to \sR$ defined as 
	\begin{align}
		\pps := \sum_{\field,\field'\in \gset}\gamma_{\field'\field} \bra{\field'\field} ,\qquad \text{where }\ \forall \field: \gamma_{\field\field}=0.
	\end{align} 
	Consequently, we have:
	\begin{align}
		\E\left[\int_0^T  \dt\ \gamma(t, \field, \field^+) \Bigm| \field_0=\field\right]=\int_{0}^{T} \dt\  \pps \lG  e^{\bG t}\ket{\field} .
	\end{align}
\end{proposition}

Now we have enough tools to tackle \eqref{eq:3-4-1}. 
Define $\tilde{\bm{L}}\colon \lH \to \sR $ as 
\begin{align}
	\tilde{\bm{L}} \ \ket{\field,\field'} = 
	\begin{cases}
		-\log G^{\field'}_{\field}, \qquad& \text{if }\ \field'\neq \field,\\
		0, & \text{otherwise}.	
	\end{cases} \label{eq:3-4-2}
\end{align}
Note that if $\field'\neq \field$ then $G^{\field'}_{\field}\geq 0$ by definition. If $G^{\field'}_{\field}= 0$, we allow $\log 0 = -\infty$ with the convention that $0\log 0 = 0$. By the above proposition, we can see that
\begin{align}
	\E\left[\int_0^T \dt\ L(w_t, w_t^+) \Bigm| \field_0=\field\right]=KT+\int_{0}^{T} \dt\  \tilde{\bm{L}} \lG  e^{Gt}\ket{\field}.
\end{align}

Before jumping into the next section, we would like to highlight one final interesting observation that links the minimization of the integral of the Lagrangian, deterministic systems, and Shannon entropy.
\begin{proposition}[\textbf{Minimizing the Integral of Lagrangian Implies Determinism}]\label{prop:min-L}
	Given a fixed frequency $\forall \field: |G^{\field}_{\field}|=K$, we can see that $\frac{G^{\field'}_{\field}}{K}$ represents a probability distribution over $\field':\field'\neq \field$. We denote Shannon entropy as
	\begin{align}
		H_\field:=-\sum_{\field':\field'\neq \field} \frac{G^{\field'}_{\field}}{K} \log\left( \frac{ G^{\field'}_{\field} }{K} \right).
	\end{align}
	Then, we have 
	\begin{align}
		\tilde{\bm{L}} \lG \ket{\field}=-K\log K+H_\field.
	\end{align}
	Therefore, the following inequality holds:
	\begin{align}
		\E\left[\int_0^T  \dt\ L(\field_t, \field_t^+) \Bigm| \field_0=\field\right]\geq KT(1-\log K),
	\end{align}
	where the minimum is achieved when the system dynamic is deterministic, i.e., every $\field$ has only one configuration $\field'$ it can jump to.
\end{proposition}

One way to interpret this result is by looking back at the path integral 
\begin{align}
	\int \gD \field \ \exp\left\{-\int_{0}^{T} \dt \ L(w_{t}, w_{t}^+)\right\},
\end{align}
where each path is assigned a weight $ \exp\{-\int_{0}^{T} \dt \ L(w_{t}, w_{t}^+)\}$. Recall that, with a constant frequency $K$, the Lagrangian is $L(w_t, w_t^+)=K-\delta({\field_t^+\neq \field_t})\cdot \log  G^{\field_{t}^+}_{\field_{t}}$. 

Suppose we vary the system dynamics while preserving the constant frequency $K$, such that the entropy $H_\field\to 0$. Then, each configuration $\field$ can almost surely jump to only one $\field'$. 
Consequently, any path not following this specific trajectory would yield a large value of  $\int_{0}^{T} \dt \ L(w_{t}, w_{t}^+)$ and thus has a negligible weight $ \exp\{-\int_{0}^{T} \dt \ L(w_{t}, w_{t}^+)\}$. Therefore, taking $H_\field\to 0$ implies that the dominant contributions to the path integral come from the path minimizing $\int_{0}^{T} \dt \ L(w_{t}, w_{t}^+)$.  This resembles the path integral formalism of quantum mechanics, where in the classical limit \(\hbar \to 0\), the path with stationary action (the integral of the Lagrangian) dominates the path integral.

%% file: sections/IF.tex
\section{\fnameU}\label{sec:IF}
The first two principles, complete configuration and locality, already give rise to a rich and complex system. It becomes more interesting as we introduce the last one, that is, endowing the system with objectives. Therefore, the goal of this section is to build the foundation for how the objective is introduced and to examine it from the perspective of artificial intelligence. 

From a perspective of artificial intelligence, such as solving machine learning tasks, it involves two aspects. The first is how  the field evolves to minimize its objective value. The second is how the objective value is designed such that its minimization leads to the desired behaviors. We discuss both aspects in this section, after introducing the necessary groundwork. 
We follow a similar approach to the previous section, beginning with a study of the entire system, then moving to the two-entity view, and finally arriving at the field formulation.

The objective value should imply the behavior of the system dynamics. Thus, whenever the system performs an action that jumps from one configuration to another, we want a signal that informs the system how good this step is. 
This is exactly what we have done in Section~\ref{subsec:complete} where a value $\ppss_{\field'\field}$ is assigned to the transition from $\field$ to $\field'$. As shown in Proposition~\ref{prop:lift-operators-a-case}, this corresponds to an operator, which we refer to as the objective operator in this case.

\begin{definition}[\textbf{Objective Operator $\pps$}]\label{def:purpose-operator}
    The objective operator is a linear operator $\pps\colon\lH\to \sR$  defined by 
    \begin{align}
		\pps  := \sum_{\field,\field'\in \gset}\gamma_{\field'\field} \bra{\field'\field},\qquad \text{where }\ \forall \field: \gamma_{\field\field}=0.
	\end{align} 
\end{definition}
Moreover, Proposition~\ref{prop:lift-operators-a-case} tells us 
\begin{align}
    \int_{0}^T \dt\ \pps \lG \ket{\varphi(t)}=\int_{0}^T \dt\ \pps \lG e^{\bG t}\ket{\varphi(0)}
\end{align}
is the expected value the system receives within time $T$ starting from state $\ket{\varphi(0)}$.
However, we want this objective value to characterize a fundamental property of the system, and it should depend on a minimal set of variables. Therefore, we define the objective value as a time average over an infinite time horizon. 
\begin{definition}[\textbf{The Averaged State $\avgfield$ and the Objective Value $\ppsv$}]\label{def:purpose-value}
    The averaged state is 
    \begin{align}
        \avgfield := \lim_{T\to \infty} \frac{1}{T} \int_{0}^T \dt\ \ket{\varphi(t)}.
    \end{align}
    The objective value associated to an objective operator is 
    \begin{align}
        \ppsv := \lim_{T\to \infty} \frac{1}{T} \int_{0}^T \dt\ \pps \lG \ket{\varphi(t)}.%
    \end{align}
    The starting state $\ket{\varphi(0)}$ will be either clear from context or irrelevant.
\end{definition}
One might wonder if the above definitions are valid in the first place. In fact, their validity comes straightforwardly from the following result.
\begin{proposition}[\textbf{The Stationary State Always Exists}]\label{prop:stationary}
    Given that the configuration set $\gset$ is finite, the stationary state 
    \begin{align}
        \ket{\varphi(\infty)}:=\lim_{t\to \infty} \ket{\varphi(t)}    
    \end{align}
    always exists.
\end{proposition}

We note that the existence of the stationary state generally requires stronger assumptions in the discrete-time case, whereas continuous time simplifies the analysis. The above result assures us that our definitions are valid. More importantly, these definitions are closely related.
\begin{corollary}[\textbf{$\avgfield$ and $\ppsv$ Always Exist}]\label{cor:avg-valid}
    Given that the configuration set $\gset$ is finite, the averaged state $\avgfield$ and the objective value $\ppsv$ always exist. Moreover, we have 
    \begin{align}
        \avgfield=\ket{\varphi(\infty)}\qquad \text{and}\qquad \ppsv=\pps \lG \avgfield.
    \end{align}
\end{corollary}

Having the necessary definitions, we formulate the principle of purposefulness as follows. 
\begin{principle}[\textbf{Purposefulness}]\label{pinciple:purpose-detail}
    Each entity $x$ in the system is associated with an objective operator $\pps_x$ and its corresponding objective value $\ppsv(x)$. Each entity evolves to minimize its objective value.
\end{principle}

Thus, the behavior of the system comes down to (1) how each entity minimizes its objective value and (2) how to design the objective operators. The rest of the section is devoted to investigating these two important questions.

Additionally, in the rest of this section, we will need the notion of ergodicity for the ease of analysis. 
\begin{definition}[\textbf{Ergodicity}]\label{def:ergodic}
    The system is ergodic if $\avgfield$ does not depend on the initial state $\ket{\varphi(0)}$.
\end{definition} 

The above property alleviates the need for excessive technical considerations in the remainder of this section. It ensures the well-definedness of many operations we will perform while maintaining a clear presentation. However, it is conjectured that only a weaker version of this property is necessary. For example, a straightforward relaxation would only require $\ppsv=\pps\lG\ket{\varphi(\infty)}$ to be independent of the initial state. This is left for future work, and we are now prepared for a deeper investigation into the minimization of the objective value.

\subsection{\fnameU: Minimizing the Objective Value}\label{subsec:min-pps}

The key question discussed in this subsection is seemingly contradictory: how does an entity learn to minimize its objective value while following the same time-invariant transition rule? Intuitively, one might consider the time-invariant behavior of ``minimization'' itself as an inherent nature of the entity, with the outcomes of this minimization representing its learned behavior. For example, consider a program designed to minimize a certain criterion; the program itself remains fixed. This implies that the entity can simulate a minimizer through fixed dynamics, thereby exhibiting objective-driven behavior through time-invariant mechanisms.

Let us first formulate such \dname\ first through the two-entity view (Section~\ref{subsec:two-entity-field}). Consider a system composed of entities $x$ and $x'$, i.e., $\gset=\gA\times \gB\times \gM\times \gN$. Focusing on entity $x$, it has an objective operator $\pps$ and the corresponding objective value $\ppsv$.  
 
\textbf{Notation.} Consider that entity $x$ is associated with a set of generators $\bM\in \genset_x$; entity $x'$ is associated with generators $\bN\in \genset_{x'}$. We can see that $\bM\in \genset_x$ and $\bN \in \genset_{x'}$ together  determine the system dynamics via $\bG=\bM+\bN$, and thus they determine the objective value $\ppsv=\pps \lG \avgfield$. We focus on entity $x$ with generator $\bM$, and thus we denote $\ppsv(\bM):=\ppsv=\pps \lG \avgfield$ as the objective value obtained by entity $x$ in this setting. 
Then, we make the following definition of what it means for dynamics to be objective-driven. 
\begin{definition}[\textbf{\dnameU}]\label{def:ID}
	In the two-entity view as described above, we say that $\bM^\star$, an entity $x$'s infinitesimal generator, corresponds to an objective-driven dynamic w.r.t. $\genset_x$ if 
	\begin{align}
		\forall \bN\in \genset_{x'}:\quad \ppsv(\bM^\star)\leq \inf_{\bM\in \genset_x}\  \ppsv(\bM), 
	\end{align} 
    where the objective operator $\pps$ only depends on $\gA\times \gB\times \gN$. 
\end{definition}
Note that $\bM^\star$ may not belong to $\genset_{x}$, as it is expected to be ``larger'' than all $\bM\in \genset_x$ in order to simulate minimization over $\genset_x$. In fact, $\bM^\star$ is only required to have compatible $\gA$ and $\gB$ for communication with the other entity, and it may possess its own configuration space $\gM^\star$.

We may view entity $x'$ as the environment and entity $x$ as an agent. The objective operator $\pps$ then describes the objective signal that the environment provides to the agent.  
Therefore, the objective-driven dynamics $\bM^\star$ can be understood as the time-invariant dynamics that learn to adapt to the environment in such a way that they minimize the given objective value.

We prove the existence of such $\bM^\star$ under certain conditions, leaving the generic case for future study. %

\begin{theorem}[\textbf{Existence of \dnameU: A Case Study}]\label{thm:ID-case}
    Consider the two-entity formulation described above for an ergodic system. If the objective operator $\pps\colon\lH \to \sR_+$ is non-negative and $\forall{\bN\in \genset_{x'}}:\min_{\bM\in \genset_x} \ppsv(\bM)=0 $ for a finite $\genset_x$, then there exists generator $\bM^\star$  that simulates the minimization, such that
    \begin{align}
        \forall \bN\in \genset_{x'}:\quad \ppsv(\bM^\star)= \min_{\bM\in \genset_x}\  \ppsv(\bM)=0. 
    \end{align}
\end{theorem}

The notion of objective-driven dynamics can be extended to the field formulation; that is, each $x\in \gX$ is associated with a local objective operator $\pps_x$, and the local dynamics $\bG(x)$ simulate the minimization of the local objective value $\ppsv(x)=\pps_x\lG\avgfield$.

The locality of the objective operator $\pps_x$ is the result of the locality principle, i.e., $\pps_x$ is only based on what happens in the local neighborhood $\nbhd_x$. 
\begin{definition}[\textbf{Local Objective Operators}]\label{def:local-pps}
    A objective operator $\pps_x  = \sum_{\field,\field'\in \gset}\gamma_{\field'\field} \bra{\field'\field}$ is local to $x\in \gX$ if 
    \begin{align}
        \field'_{|\nbhd_x}=\field_{|\nbhd_x}\quad \implies\quad \gamma_{\field'\field} = 0.
    \end{align}
\end{definition}

Thus, we can formulate the \fname\ as the following.
\begin{definition}[\textbf{\fnameU}]\label{def:IF}
	Consider a dynamical stochastic field where each entity $x\in \gX$ is associated with a local objective operator $\pps_x$ and a set of local infinitesimal generators $\genset_x$. A \fname\ is formed by \dname\ w.r.t. $\genset_x$ that simulates the ``minimization'' of the local objective value $\ppsv_x=\pps_x\lG\avgfield$.
\end{definition}

How should each entity minimize its local objective value? In the proof of Theorem~\ref{thm:ID-case}, we construct such minimization dynamics via random search, i.e.,  it simply simulates random trials of different choices of $\bM\in \genset_x$. It is reasonable to expect that there are more clever ways. In particular, it would be interesting if $\bM^\star$ could simulate gradient descent on the objective value $\ppsv$. That is, $\bM^\star$ maintains an $\bM(t)\in \genset_x$ and simulates
\begin{align}
    \frac{\diff \bM}{\diff t}=-\frac{\partial \ppsv(\bM)}{\partial \bM}.
\end{align} 

The immediate question is: What exactly is this gradient? Next, we address this question by providing the explicit formula for the gradient.

\subsection{The Gradient Formula}\label{subsec:grad}
In this subsection, we derive the gradient formula for the two-entity view,  but first, we need some necessary building blocks.

Let us imagine the system suddenly varies the transition rate $G^{\field'}_{\field}$ slightly. This change would affect the probability that the system jumps to configuration $\field'$ from configuration $\field$. Therefore, it should be related to $\ket{\field'\field}\in \lH$ in some ways. It turns out to be useful to first define the following projection operator and its resulting subspace of $\gH$.

\begin{definition}[\textbf{Projection Operator $\proj\colon\lH\to \subH$} ]\label{def:proj}
    The operator $\proj$ is a linear operator defined by
    \begin{align}
        \proj \ \ket{\field'\field} = \ket{\field'}-\ket{\field}.  \label{eq:proj}
    \end{align} 
    The resulting subspace $\subH\subset \gH$ is denoted by 
    \begin{align}
        \subH:= \proj\ \lH.
    \end{align}
\end{definition}
The subspace $\subH$ is the span of vectors of the form  $\ket{\field'}-\ket{\field}$. Moreover, it admits an alternative representation, as shown in the following result. 
\begin{fact}[\textbf{Another Representation of $\subH$}]\label{fact:subH}
Write $\sfield=\sum_{\field}\sfcoef^\field \ket{\field}$, and we have
    \begin{align}
        \subH=\proj\ \lH=\left\{\sfield \in \gH\mid \sum_{\field}\sfcoef^\field =0 \right\}.
    \end{align}
\end{fact}

There is an important fact about the projection operator: it projects the lifted generator $\lG\colon\gH\to \lH$ back to the original generator $\bG$.
\begin{fact}\label{fact:proj-G}
    By definition 
	\begin{align}
		\bG=\proj\ \lG.
	\end{align}    
\end{fact}
Therefore, we can see that the generator $\bG\colon\gH\to \subH$ maps from $\gH$ to its subspace $\subH$.

Next, given that $\subH$ involves the difference for a transition step at the current moment, it is essential to examine the long-term behavior of such a step. The following operator plays a crucial role in capturing this long-term behavior.

\begin{definition}[\textbf{The Operator $\bS$}]\label{def:S}
    Define linear operator $\bS\colon\subH\to \gH$ as 
    \begin{align}
        \bS:= \int_{0}^{\infty}\dt \ e^{\bG t}.
    \end{align}
\end{definition}
It is not immediately obvious whether the definition above is valid, since $\bS$ involves integration to infinity without any explicit normalization to keep it finite. However, we can be assured if the system is ergodic (Definition~\ref{def:ergodic}). 
\begin{proposition}\label{prop:bound-S}
    The operator $\bS\colon\subH\to \gH$ is bounded if the system is ergodic. 
\end{proposition}

\begin{figure}[t]
\centering
\begin{tikzcd}[row sep=2.5em, column sep=3.5em]
\subH \arrow[r, hookrightarrow] 
& \gH 
  \arrow[r, "\lG"] 
  \arrow[dr, swap, "\bG"] 
& \lH \arrow[d, "\proj"] \\[1em]
& & \subH \arrow[r, "\bS"] & \gH
\end{tikzcd}

\caption{A commutative diagram illustrating the relationships among spaces $\subH, \gH, \lH$, and linear operators $\bG, \lG, \proj$, and $\bS$.}
\label{fig:diagram-relations}
\end{figure}

Figure~\ref{fig:diagram-relations} provides a visual summary of the relationships among the spaces and operators discussed above.

There is one last step before we can calculate the gradient: we must first clarify the parameter with respect to which we are differentiating. Recall that we define the infinitesimal generators $\bG$ by enumerating all possible behaviors of this linear operator, i.e., $\bG=\sum_{\field\field'}\ket{\field'}G^{\field'}_{\field}\bra{\field}$ (Definition~\ref{def:inf-generator}). However, for it to be considered a valid infinitesimal generator, it must satisfy a constraint $\sum_{\field'}G^{\field'}_{\field}=0$. 

To incorporate this constraint in a way that allows us to safely take the gradient, we consider the following reparameterization of the generator, using the fact that $\bG=\proj \lG$. 

\begin{definition}[\textbf{Parameterization of the Generators}]\label{def:reparam-generator}
    We consider the following parameterization of the generator $\bG$.
    \begin{align}
        \bG=\sum_{\field,\field'\in \gset} \proj\ \act_{\field'}^{\field}G^{\field'}_{\field},
    \end{align}
    where $\forall \field: G^{\field}_{\field}$ is irrelevant and thus can be ignored and $\act_{\field'}^{\field}$ is the linear operator defined in \eqref{def:action-operator}. 
\end{definition}

Intuitively, the parameterization above can be understood as follows. Suppose we vary $G^{\field'}_{\field}$ in $\bG=\sum_{\field\field'} \proj\ \act_{\field'}^{\field}G^{\field'}_{\field}$. This variation would be equivalent to varying both $G^{\field'}_{\field}$ and $G^{\field}_{\field}$ in $\bG=\sum_{\field\field'}\ket{\field'}G^{\field'}_{\field}\bra{\field}$, at the same time, such that $\bG$ remains a valid generator. In other words, the parameterization above automatically satisfies the constraint  $\sum_{\field'}G^{\field'}_{\field}=0$.

With enough building blocks, we can finally calculate the gradient of the objective value $\ppsv=\pps\lG\avgfield$ w.r.t. the generator $\bG$.
\begin{proposition}[\textbf{Gradient Formula}]\label{prop:grad}
    Given an objective operator $\pps\colon\lH\to \sR$, assuming the system is ergodic, with the parameterization given by Definition~\ref{def:reparam-generator} we have 
    \begin{align}
        \frac{\partial \ppsv}{\partial G^{\field'}_{\field}}=\pps (1+\lG \bS \proj)\act_{\field'}^{\field} \avgfield.
    \end{align}
\end{proposition}

The above proposition tells us how the objective value $\ppsv$ changes given an infinitesimal variation in the generator. Thus, in the two-entity view, using the decomposition theorem $\bG=\bM+\bN$ (Proposition~\ref{prop:decomposition}), we can find the gradient of the generator of the entity with generator $\bM$. 

In order to do this, similar to how we lift $\bG$ to $\lG$, we would need to lift $\bM$ to $\lM$ and $\bN$ to $\lN$. This can be done pretty straightforwardly via the following operations. 

\textbf{Notation.} We denote $\act_{\alpha'\mu'}^{\alpha\beta\mu}\colon\gH\to \lH$ as a linear operator that transforms $\ket{\alpha\beta\mu\nu}$ to $\ket{{\alpha'\beta\mu'\nu}\ {\alpha\beta\mu\nu}}$, regardless of $\nu$. Formally
\begin{align}
	\act_{\alpha'\mu'}^{\alpha\beta\mu}\ket{\alpha''\beta''\mu''\nu}=
	\begin{cases}
		\ket{{\alpha'\beta\mu'\nu}\ {\alpha\beta\mu\nu}},\qquad &\text{if}\ \alpha\beta\mu=\alpha''\beta''\mu''\\
		0, &\text{otherwise}.
	\end{cases} \label{def:action-operator-M}
\end{align}
Similarly, we define $\act_{\beta'\nu'}^{\alpha\beta\nu}$ for the other entity.
\begin{definition}[\textbf{ Lift of the Generator $\lM, \lN$}]\label{def:lift-generator-entity}
    The lift of generator $\lM, \lN\colon\gH\to \lH$ are defined to be
    \begin{align}
        \lM=\sum_{\alpha\beta\mu\alpha'\mu'}M^{\alpha'\mu'}_{\alpha\beta\mu}\act_{\alpha'\mu'}^{\alpha\beta\mu},\qquad \text{and}\qquad \lN=\sum_{\alpha\beta\nu\alpha'\nu'}N^{\beta'\nu'}_{\alpha\beta\nu}\act_{\beta'\nu'}^{\alpha\beta\nu}.
    \end{align}
\end{definition}
Since $\bG=\bM+\bN$ (Proposition~\ref{prop:decomposition}), by lifting $\bM$ and $\bN$ we also lift $\bG$, i.e., $\lG=\lM+\lN$.
\begin{fact}\label{fact:lifted-decompose}
    With the above definition, we have
    \begin{align}
        \bM=\proj\ \lM, \qquad \bN=\proj\ \lN.
    \end{align}
\end{fact}

\begin{proposition}[\textbf{Gradient Formula in the Two-entity View}]\label{prop:grad-entity}
    Given an objective operator $\pps\colon\lH\to \sR$, assuming the system is ergodic, with the parameterization of generator $\bM$ given by Definition~\ref{def:lift-generator-entity} we have 
    \begin{align}
        \frac{\partial \ppsv}{\partial M^{\alpha'\mu'}_{\alpha\beta\mu}}=\pps (1+\lG \bS\proj)\act_{\alpha'\mu'}^{\alpha\beta\mu} \avgfield.
    \end{align}
\end{proposition}

Finally, we turn to the field formulation. It is easy to see that, given an entity $x\in \gX$, we may view the rest of the field as the other entity,  thus reducing the system to a two-entity view. Therefore, the action operator $\act_{\alpha'\mu'}^{\alpha\beta\mu}$ can be reformulated directly using the field notation. The important fact is that the action operator $\act(x)_{\field'}^{\field}$ changes the local configuration $\field_{|\gset_x}$ to $\field'_{|\gset_x}$. Similar to previous discussions, given an arbitrary configuration $\field''$, $\act(x)_{\field'}^{\field}$ would act on $\ket{\field''}$ only if $\field''$ and $\field$ agree on the local neighborhood $\nbhd_x$. Then, the action would leave $\gset/\gset_x$ unchanged, resulting in   $\field'_{|\gset_x}\times \field''_{|\gset/\gset_{x}}$. Note that we adopt the convention of automatic ordering: $ \field''_{|\gset/\gset_{x}}\times \field'_{|\gset_x}$ refers to the same configuration as  $\field'_{|\gset_x}\times \field''_{|\gset/\gset_{x}}$. This action operator is formally defined as follows.

\textbf{Notation.} We denote $\act(x)_{\field'}^{\field}\colon\gH\to \lH$ as a linear operator that transforms $\ket{\field}$ to $\ket{{\field^+}{\field}}$, where $\field^+=\field'_{|\gset_x}\times \field''_{|\gset/\gset_{x}}$, given that $\field$ and $\field'$ agree on $\gset/\gset_x$. Formally, it is defined as:
\begin{align}
	\act(x)_{\field'}^{\field}\ket{\field''}=
	\begin{cases}
		\ket{\field^+\field},\quad &\text{if}\ \field_{|\nbhd_x}=\field''_{|\nbhd_x} \ \text{and}\ \ \field'_{|\gset/\gset_x}=\field_{|\gset/\gset_x}\\
		0, &\text{otherwise}.
	\end{cases} \label{def:action-operator-field}
\end{align}
Therefore, the local generators can be formulated as 
\begin{align}
    \bG(x)=\sum_{\field,\field'\in \gset}\proj\  G(x)^{\field'}_{\field} \act(x)_{\field'}^{\field}.
\end{align}
Having all the building blocks, the field version of the gradient formula is stated as follows.

\begin{theorem}[\textbf{Gradient Formula in the Field}]\label{thm:grad-field}
    Given a local objective operator $\pps_x\colon\lH\to \sR$, assuming the system is ergodic, the gradient formula for $\ppsv(x)=\pps_x \lG \avgfield$ w.r.t. local generator $\bG(x)$ is
    \begin{align}
        \frac{\partial \ppsv(x)}{\partial G(x)_{\field}^{\field'}}=\pps_x (1+\lG \bS\proj) \act(x)^{\field}_{\field'} \avgfield.
    \end{align}
\end{theorem}

The objective operators determine the behavior of the system, as each entity minimizes its objective value. However, if we want our field to perform some tasks, it is unclear what operators could provide us with the desired result, especially for such a complex,  dynamical, stochastic, and decentralized system. This is of great importance and is our immediate focus. 

\subsection{How the Objective Operators may be Designed}\label{subsec:pps-design}

We begin with the two-entity formulation $\gset=\gA\times\gB \times \gM\times \gN$ where the objective operator is the easiest to make sense of. Let us view one entity $x_e$ as the ``environment'' and the other entity $x_a$ as an ``agent'' that interacts with the environment. We focus on the entity $x_a$ where it minimizes its objective value given by an objective operator $\pps$. Note that the objective operator $\pps$ needs to be local, i.e., it only depends on what the entity $x_a$ observes. 

Next, imagine dividing the configuration space $\gM$, the private configuration of the agent, into two parts. The ``outer'' part keeps the same communication channel to the environment through $\gA$ and $\gB$, while the ``inner'' part only communicates with the outer part. Thus, the entity $x_a$ is composed of two entities, $x_1$ and $x_2$, where $x_1$ corresponds to the outer part connected to the environment $x_e$, and $x_2$ corresponds to the inner part that only connects to $x_1$. The corresponding directed graph $\gX$ is illustrated as: 
\begin{align}
    x_e \leftrightarrows x_1 \leftrightarrows x_2.
\end{align}

The objective operator $\pps$ now provides objective signals only to the outer entity $x_1$, leaving no objective operator to guide the inner entity $x_2$. In fact, from $x_2$'s perspective, its environment comprises both $x_e$ and $x_1$. Therefore, since $x_1$ is the only component of $x_2$'s environment that can communicate with $x_2$, it should provide $x_2$ with its objective signals through another objective operator $\pps'$.

The trivial choice for the outer entity $x_1$ would be to pass the objective value it receives directly to $x_2$, i.e., by setting $\pps'$ to be the same as $\pps$. We note that it is not really possible to have $\pps'=\pps$ be identical operators, as it would require the objective signals to be transmitted instantaneously, thus violating locality. Instead, there is a response time for this transition. Nevertheless, as long as the signal passing is not lost, the corresponding objective values will be preserved. Therefore, we adopt the following notation.

\textbf{Notation.} We use $\pps'\leftarrow\pps$ to denote
that the objective signal given by $\pps$ appears as an objective signal given by $\pps'$ after a finite response time. The important property we need is that they induce the same objective value, i.e., $\pps'\lG \avgfield = \pps \lG \avgfield$, as a consequence of not missing propagating signals.

It is not really satisfying if the objective value can only be passed unchanged. Indeed, there is nothing to stop the outer entity $x_1$ from doing something about it. Let us denote this ``something'' as a linear operator $\ppsp:\lH\to \lH$, and thus
\begin{align}
    \pps'\leftarrow \pps \ppsp.
\end{align}
The design comes down to the choice of the $\ppsp$ operator, which we refer to as an \textit{objective propagator}. Note that, although $\ppsp$ is a linear operator, it can assign an arbitrary value for any jump $\field\to \field^+$, as long as the original objective operator $\pps$ is not zero.
The outer entity $x_1$ can choose arbitrary $\ppsp$ operators, placing the inner entity $x_2$ at its mercy, as the inner entity is obliged to minimize whatever objective value the outer entity generates. 

Here is a crucial observation. In order for the inner and outer entities, when viewed as a whole, to minimize the original objective value given by $\pps$, the objective propagation operators $\ppsp$ must be chosen carefully. Given that the inner entity is obliged to minimize its objective value $\ppsv'=\pps' \lG \avgfield=\pps\ppsp \lG \avgfield$, it should align with the original objective value $\ppsv=\pps \lG \avgfield$ (up to a constant shift). This means that we would want  $\lG\avgfield$ to be a fixed point of the operator $\ppsp$, i.e.,
\begin{align}
    \ppsp\lG\avgfield =\lG\avgfield.
\end{align}

Now, let us see the world from entity $x_1$'s perspective, where $x_e$ and $x_2$ form its environment. Without any additional assumptions, there is no way for $x_1$ to tell whether $x_2$ is also connected to $x_e$. Therefore, given the symmetry, it is reasonable to expect that entity $x_2$ also generates objective signals to entity $x_1$. As a result, entity $x_2$ would also choose an objective propagation operator $\ppsp'$ and generate the objective signal $\pps''$ to provide to entity $x_1$, i.e.,
\begin{align}
    \pps''\leftarrow \pps'\ppsp'.
\end{align}
To balance the two objective signals, let us simply introduce some weights for them, thus making the final objective operator for entity $x_1$ a weighted average of $\pps$ and $\pps''$. 

Suppose we keep doing this division; we would then arrive at a field formulation where each local objective operator is given by its neighborhood. This nested behavior is elaborated in the next section.

\subsection{Objective Propagation}\label{subsec:ppsp}
We begin by introducing a quantity that defines the connection strength of points in the space $\gX$. It will be used to synthesize objective signals, as hinted in Subsection~\ref{subsec:pps-design}.
\begin{definition}[\textbf{Adjacency Weights $\ppsadj$}]\label{def:adj-weight}
    For each $x,x'\in \gX$, there is a weight $\ppsadj^{x'}_{x}\in[0, 1]$ that represents the connection strength from $x'$ to $x$. Note that we require $\sum_{x'}\ppsadj^{x'}_{x}=1$, and $\ppsadj^{x'}_{x}=0$ if $x'$ does not connect to $x$. 
\end{definition}
The adjacency weights are used to compute a weighted average of propagated objective signals, which are defined by operators.
However, the objective propagation operators are not designed without any requirements. In particular, as discussed in the previous subsection, there are some environmental entities that generate the ``true'' objective signal. Therefore, it is necessary to differentiate between such environmental entities, where we have no control, and the acting entities, whose behavior we aim to design.

\begin{definition}[\textbf{Environmental Entity and Acting Entity}]\label{def:env-act-entity}
    The ground space $\gX=\gX_e \sqcup \gX_a$ is the disjoint union of the set $\gX_e$ of environmental entities and the set $\gX_a$ of acting entities. The objective operators are defined over the acting entities.   
\end{definition}

Then, each acting entity is associated with an objective propagator which determines how it deals with the objective signals. 
\begin{definition}[\textbf{Objective Propagators $\ppsp$}]\label{def:ppsp}
    For acting entities $x,x'\in \gX_a$, there is a linear operator $\ppsp_{x'x}\colon\lH\to \lH$ that defines how entity $x'$ generates objective signals to a neighboring entity  $x$, where $\ppsp_{x'x}=0$ if $x'$ does not connect to $x$. Moreover, since each entity $x$ can only observe its neighborhood $\nbhd_x$, it requires the propagation operators $\ppsp_{xx'}$ to be local to $x$, i.e., 
    \begin{align}
        \forall x, x': \quad \field'_{|\nbhd_x}=\field_{|\nbhd_x}\quad \implies\quad \ppsp_{xx'}\ket{\field'\field} = 0
    \end{align} 
\end{definition}

In designing the objective operators for the acting entities, as discussed in the previous subsection, we expect $\ppsp\lG\avgfield=\lG\avgfield$, such that the propagated objective signals stay meaningful and aligned with the ``true'' objective signal given by the environmental entities. Observe that this would require the propagation operators $\ppsp$ to depend on the field dynamics $\bG$, except in the trivial case $\ppsp=\mathbf I$.

Therefore, the objective operator of each acting entity is given by objective signals propagated from other entities. The propagation is given by  $\pps_x\leftarrow \sum_{x'}\ppsadj^{x'}_{x}\pps_{x'x}$ where $\pps_{x'x}=\pps_{x'}\ppsp_{x'x}$. For each entity, we may picture this operation as both propagating and generating objective signals to its neighbors.  The final effect of objective propagation can be effectively captured by the following equations.  

\begin{definition}[\textbf{Objective Propagation Equation}]\label{def:ppsp-eq}
    The objective operator for $x\in \gX_a$ is represented by 
    \begin{align}
        \pps_x= \sum_{x'}\ppsadj^{x'}_{x}\pps_{x'x}, \qquad \text{where } \ \pps_{x'x}=\pps_{x'}\ppsp_{x'x}, \forall x'\in \gX_a,
    \end{align}
    or $\pps_{x'x}$ is given by environmental an entity $x'\in \gX_e$. 
\end{definition}
Note that this results in a non-local objective operator $\pps_x$ serving as an effective objective operator to characterize what an entity $x$ will eventually receive given any jump, possibly non-local to $x$, in the field. This simplifies the calculation and ensures the same $\pps_x \lG \avgfield= \sum_{x'}\ppsadj^{x'}_{x}\pps_{x'x}\lG \avgfield$ as desired. Another perspective is to consider the speed of objective propagation to be much faster than the propagation of field actions, resulting in a seemingly instantaneous non-local objective propagation.
For notational ease, we use the same $\pps_x$ to denote both the original local objective operator and the corresponding non-local effective objective operator. However, in the rest of the paper, we will exclusively use $\pps_x$ to refer to the effective objective operator that satisfies the aforementioned propagation equation.

Intuitively, the objective propagation should preserve the same objective value, given $\ppsp_{x'x}\lG\avgfield=\lG\avgfield$ and  $\sum_{x'}\ppsadj^{x'}_{x}=1$. We can show that this is indeed the case, provided that the acting entities are connected strongly. 
\begin{definition}[\textbf{Strongly Connected Graph}]\label{def:strong-connected}
    A directed graph $\gX$ is strongly connected if $\forall x, x'\in \gX$ there is a path, which may include multiple other nodes,  from $x$ to $x'$.
\end{definition}
\begin{proposition}[\textbf{Local Objective Values}]\label{prop:local-ppsv}
    If the objective propagation satisfy the following conditions:
    \begin{itemize}
        \item $\ppsp_{x'x}\lG\avgfield = \lG\avgfield$ \ for every acting entity $x'$ connecting to acting entity $x$;
        \item $\pps_{x'x}\lG \avgfield =\ppsv$ \ for every environmental entity $x'$ connecting to acting entity $x$;
        \item $\gX_a$ is strongly connected;
    \end{itemize}
    then every acting entity has the same objective value $\ppsv$, i.e.,
    \begin{align}
        \forall x\in \gX_a: \ppsv(x)=\pps_x \lG \avgfield=\ppsv.
    \end{align}
\end{proposition}

However, there would be a crucial problem if we attempt to calculate the gradient of $\ppsv$ w.r.t. the local generators. For example, let us impose an infinitesimal variation $\delta \bG(x)$ to the local generator at $x$. This would induce a variation $\delta \avgfield$ in the system's stationary state $\avgfield$. Moreover, since each objective propagator $\ppsp_{x}$ depends on the system dynamics, it would also incur a variation $\delta \ppsp_{x}$ for all entities. Since each objective operator $\pps_x$ depends on the objective propagators, there would also be a variation $\delta \pps_x$ in the effective objective operator. Therefore, this results in an incurred variation in the objective value that entity $x$ receives.
\begin{align}
    \delta \ppsv(x) = \delta(\pps_x \lG\avgfield)= \pps_x \delta(\lG\avgfield) + \delta(\pps_x) \lG\avgfield.
\end{align}

There are two terms arising from $\delta \ppsv$ as shown above. The first term, $\pps_x \delta(\lG\avgfield)$, can be seen as the case where the objective operators are fixed. Thus, it reduces to the scenario previously studied, and the corresponding result is essentially given by Theorem~\ref{thm:grad-field}. However, the second term involves $\delta(\pps_x)$, which requires solving the following global equation derived from the objective propagation equation (Definition~\ref{def:ppsp-eq}):
\begin{align}
    \delta \pps_x = \sum_{x'}\ppsadj^{x'}_{x} \delta \pps_{x'x}, \qquad \text{where}\ \ \delta \pps_{x'x}=\delta \pps_{x'}\ppsp_{x'x} + \pps_{x'}\delta \ppsp_{x'x}.
\end{align}
Although it may be possible to solve this by some clever local computations, we find an interesting class of objective propagators which eliminates the need of doing this. Furthermore, this class of objective propagators possesses several other very interesting properties, which are introduced in the following subsection.

\subsection{An Interesting Choice of the Propagation Operator $\ppsp[\ppsq]$}\label{subsec:ppsp-choice}

The motivation for this propagator is simple: instead of just passing the objective signal untouched, let us also convey some sort of gradient information alongside it. Interestingly, it turns out that it leads to local gradient computation. 

The gradient formula (Theorem~\ref{thm:grad-field}) shows that the gradient of the local generator for entity $x$ is given by 
\begin{align}
    \pps (1+\lG \bS\proj) \act(x)^{\field}_{\field'} \avgfield.
\end{align}
We can observe that the operator $(1+\lG \bS\proj)$ appears just as an objective propagator would. 

\textbf{Notation.} Thus, let us denote 
\begin{align}
    \ppsp=1+\lG \bS\proj. 
\end{align}

Note that Proposition~\ref{prop:stationary} ensures the existence of $\lim_{t\to \infty} e^{\bG t}$. We denote $\infevo$ as the operator of this infinite time evolution, i.e., 
\begin{align}
    \infevo:=\lim_{t\to \infty} e^{\bG t}.
\end{align}
Then, the following lemma is the core of how $\ppsp$ can exhibit many interesting properties.
\begin{lemma}\label{lemma:SPiG}
    The following equation stands. 
    \begin{align}
        1+\bS \proj \lG=1+\proj \lG \bS =\infevo.
    \end{align} 
\end{lemma}

It turns out that $\ppsp$ is a well-defined objective propagator with the following desirable properties. 
\begin{proposition}\label{prop:P-properties}
    The operator $\ppsp=1+\lG \bS\proj$ satisfies the following properties.
    \begin{enumerate}
        \item $\ppsp \lG \avgfield=\lG\avgfield$.
        \item $\ppsp^2=\ppsp$.
        \item Given a variation $\delta \bG$, the incurred $\delta \ppsp$ satisfies $\delta \ppsp \lG \avgfield=0$.
    \end{enumerate}
\end{proposition}

However, there is a problem: the above propagator $\ppsp$ is not local, meaning that it would react to any dynamics within the system. To ensure it only reacts to local dynamics, we may insert a local operator $\ppsq:\gH\to \gH$ to filter out distant dynamics. As we will see, this operator needs to be closed on $\subH$, i.e., $\ppsq \subH \subseteq \subH$. In fact, such operators are easy to find. For example, any linear operator that preserves the normalization will be closed on $\subH$.  As shown below, this little modification preserves essentially all the desirable properties of  $\ppsp$.

\begin{definition}[\textbf{Objective Propagators $\ppsp[\ppsq]$}]\label{def:ppsp-q}
    Given linear operator $\ppsq\colon\subH\to \subH$, we define objective propagators $\ppsp[\ppsq]\colon\lH\to \lH$ as 
    \begin{align}
        \ppsp[\ppsq]:=1+\lG \bS \ppsq \proj .
    \end{align}
\end{definition}

Any operator $\ppsq$ that is closed on $\subH$ would ensure that $\ppsp[\ppsq]$ has the same desirable properties as $\ppsp$, which we will derive following the next definition.
\begin{definition}[\textbf{Bilinear Map $(\cdot, \cdot)$}]\label{def:bi-map}
    Given two operators $\ppsq, \ppsq'\colon\subH\to \subH$, we define a bilinear map $(\cdot, \cdot)$ as the following.
    \begin{align}
        (\ppsq,\ppsq'):=\ppsq+\ppsq'-\ppsq\ppsq'.
    \end{align}
\end{definition}

\begin{proposition}\label{prop:PQ-properties}
    The operator $\ppsp[\ppsq]=1+\lG \bS \ppsq \proj$ satisfies the following properties.
    \begin{enumerate}
        \item $\ppsp[\ppsq] \lG \avgfield=\lG\avgfield$.
        \item $\ppsp[\ppsq]\ppsp[\ppsq']=\ppsp[(\ppsq,\ppsq')]$.
        \item Given a variation $\delta \bG$, the incurred $\delta (\ppsp[\ppsq])$ satisfies $\delta (\ppsp[\ppsq]) \lG \avgfield=0$.
    \end{enumerate}
\end{proposition}

Having enough tools at hand, we can show that with this choice of objective propagators, we have $\delta (\pps_x)\lG \avgfield=0$, and thus the gradient of $\ppsv(x)=\pps_x\lG \avgfield$ w.r.t. local generator $\bG(x)$ can be computed locally. 

\begin{theorem}[\textbf{$\ppsp[\ppsq]$ Allows for Local Gradient Computations }]\label{thm:ppsq-grad}
    When the acting entities $\gX_a$ are strongly connected, objective propagators of the form of $\ppsp_{x'x}=\ppsp[\ppsq_{x'x}]$ result in local gradient computations. That is,  Theorem~\ref{thm:grad-field} applies, and we obtain the same gradient formula 
    \begin{align}
        \frac{\partial \ppsv(x)}{\partial G(x)^{\field'}_{\field}}=\pps_x\ppsp\act(x)_{\field'}^{\field} \avgfield.
    \end{align}
\end{theorem}
Theorem~\ref{thm:ppsq-grad} and Proposition~\ref{prop:PQ-properties} together reveal several interesting properties of the propagator $\ppsp[\ppsq]$. In particular, this propagator has a similar form to the operator $\ppsp$ that appears in the gradient formula, suggesting that it may encode certain gradient-related information. This connection could potentially be beneficial for estimating the gradients of local entities. However, its precise implications remain unclear and are left for future work.

%% file: sections/conclusion.tex
\section{Discussion}\label{sec:discussion}

Let’s first review the main results of this paper before exploring the various perspectives one might connect them to.

\subsection{Summary of the Main Results}\label{subsec:summary-results}

We start from the three fundamental principles: complete configuration, locality, and purposefulness. 
These principles lead to a framework for objective-driven dynamical stochastic fields, which we refer to as intelligent fields. An objective-driven dynamical stochastic field can be conceptualized as a collection of entities, each possessing configurations that may not be entirely visible to other entities, while sharing some parts of their configurations with their neighbors. Each entity generates and sends objective signals to the neighboring entities while simultaneously aiming to minimize the objective value given by its neighbors. Thus, the system evolves on its own, driven by the objective propagation. We may view some parts of the field as environmental entities, while others are considered to be acting entities. The design of the mechanism of objective propagation for the acting entities is critical to ensuring the system operates as intended, such as in solving AI-related tasks.

However, in order to study the behavior of the field such that we can design the mechanism of objective propagation, we need the necessary machinery for formal analysis. This leads to our theoretical framework for objective-driven dynamical stochastic fields, which is the main content of this paper. 

The first two principles, complete configurations and locality, lead to a compact model for dynamical stochastic fields, as shown in Section~\ref{sec:field}. The main result of this section can be summarized by the equations below. 
\begin{align}
    \bG=\sum_x \bG(x),\qquad L(w_{t}, w_{t}^+)=\sum_x L(w_{t, x}, w_{t,x}^+), \\
    \bra{\field'} e^{\bG T} \ket{\field}=\int \gD \field \ \exp\left\{-\int_{0}^{T} \dt \ L(w_{t}, w_{t}^+)\right\}.
\end{align}
That is, the infinitesimal generator $\bG$ of the whole system is the sum of local infinitesimal generators  $\bG(x)$. Notably, the local generators commute, i.e., $[\bG(x), \bG(x')]=0$, if $x$ and $x'$ are not neighbors. Similarly, the Lagrangian $L(w_{t}, w_{t}^+)$ of the entire system is the sum of the local field Lagrangian $L(w_{t,x}, w_{t, x}^+)$. The transition probability $\bra{\field'} e^{\bG T} \ket{\field}$ can be formulated in a path integral formalism using the Lagrangian. 
As we take the limit that the entropy of the system $H_w\to 0$, the path that minimizes the integral of the Lagrangian becomes the dominant contribution to the path integral. 

In Section~\ref{sec:IF}, we introduce the principle of purposefulness. That is, each entity $x$ is associated with an objective operator $\pps_x$, and it evolves to minimize its objective value: 
\begin{align}
    \ppsv(x)=\lim_{T\to \infty} \frac{1}{T} \int_0^T \dt\ \pps_x \lG \sfield(t),
\end{align}
which is the time average of the objective signals it receives over an infinite time horizon. We call an entity objective-driven if it evolves to minimize its objective value. To construct such objective-driven dynamics, we may simulate random trials of different local generators and select the most effective one. However, a more intriguing approach would be to enable the system to simulate gradient descent. This leads to the key question: What is the gradient of the local objective value with respect to the local generator? We derive this gradient formula as 
\begin{align}
    \frac{\partial \ppsv(x)}{\partial G(x)_{\field}^{\field'}}=\pps_x \ppsp \act(x)^{\field}_{\field'} \avgfield, \qquad \text{where }\ \ppsp=1+\lG \bS\proj.
\end{align}
The above gradient formula stands if the objective operator $\pps_x$ is fixed. However, in the scenario where each entity exchanges objective signals with its neighbors, $\pps_x$ would depend on the system dynamics. This dependence is defined by the objective propagator $\ppsp_{x'x}$ via the objective propagation equation: 
\begin{align}
    \pps_x= \sum_{x'}\ppsadj^{x'}_{x}\pps_{x'x}, \qquad \text{where } \ \pps_{x'x}=\pps_{x'}\ppsp_{x'x}.
\end{align}
However, if the objective propagator is not trivial, the gradient formula would be non-local in general, making it difficult for local entities to estimate its corresponding gradient. We find an interesting class of objective operators 
\begin{align}
    \ppsp[\ppsq]=1+\lG \bS\ppsq\proj,
\end{align}  
which is in a similar form to the operator $\ppsp=\ppsp[\mathbf I]$ appearing in the gradient formula. This suggests that objective propagation with $\ppsp[\ppsq]$ may inherently convey certain types of gradient information. Interestingly, this objective propagator allows for local gradient computation, where the resulting gradient formula coincides with that of the previously discussed trivial case.

\subsection{Discussion on Different Perspectives}\label{subsec:literature}

The proposed framework for objective-driven dynamical stochastic fields is derived in a self-contained manner, from three fundamental principles: complete configuration, locality, and purposefulness. These principles are intuitive, echoing concepts found across diverse disciplines. As a result, the framework is potentially interdisciplinary and can be viewed from various perspectives. Given its generality, it connects to a vast body of existing literature.  Although a comprehensive overview is beyond our current scope, we acknowledge several major connections. We note, however, that these high-level connections should be viewed as suggestive rather than definitive.

\paragraph{Neural Networks.}
An objective-driven dynamical stochastic field may be conceptualized as a kind of neural network within an environment. From this perspective, the neural network consists of acting entities in the field, while the environment comprises environmental entities. Thus, each acting entity is viewed as a neuron, collectively forming a neural network. These neurons have individual configurations that determine their subsequent actions, including changing their own configuration and interacting with other entities. Similar to how neurons in traditional neural networks transmit signals to connected neurons, entities in this field model the exchange of signals that influence each other's behavior. This results in a dynamic and adaptive neural network, where each neuron's behavior continuously adjusts based on the objective signals it receives from the other entities, including those from the environment. Each neuron also generates and propagates objective signals to other neurons, resembling a feedback propagation mechanism similar to those found in various artificial neural networks.

The prominent models of artificial neural networks (ANNs) are deep neural networks, which are commonly trained via backpropagation~\citep{rumelhart1986learning, lecun1988theoretical, hecht1992theory}. 
Backpropagation is considered to be non-local in nature, which some argue lacks biological plausibility~\citep{grossberg1987competitive, bengio2015towards}, and efforts are made to train deep neural networks in more biologically plausible ways~\citep{bengio2015towards, hinton2022forward}.
Unlike deep neural networks, some ANNs are naturally biologically plausible. Notable examples include models like the Hopfield network~\citep{hopfield1982neural} and Boltzmann machines~\citep{ackley1985learning}, which are designed for storing information through associative memory. Both of these were influenced by local Hebbian learning theory and are conventionally trained via minimizing certain energy functions. Thus, from the perspective of objective-driven dynamic stochastic fields, this energy function can be viewed as a special case of an objective value, one without dynamics or propagation, that the system aims to minimize.  
Furthermore, spiking neural networks~\citep{maass1997networks, gerstner2002spiking, yamazaki2022spiking} offer another layer of biological realism by mimicking the actual electrical activity of neurons, which can also be trained via local Hebbian learning rules. 

However, conventional ANNs, within the typical statistical learning setting, do not explicitly account for dynamic environments, as they primarily learn from static datasets. In contrast, reinforcement learning addresses this by explicitly incorporating dynamic interactions with an environment, as we discuss below.

\paragraph{Reinforcement Learning.}
Reinforcement learning (RL) has evolved significantly since its establishment, starting from early frameworks focusing on single agents in fully observable Markov decision processes (MDPs), where an agent receives scalar reward signals directly tied to its performance~\citep{sutton1988learning, sutton2018reinforcement}. Building on this foundational setting, multi-agent RL emerged to address scenarios with multiple agents learning simultaneously, each receiving its own reward signals while interacting in shared or competitive environments~\citep{hu1998multiagent, panait2005cooperative, bansal2017emergent, 2024_Jin}. In parallel, partially observable RL tackled problems where the agent’s observations are incomplete, thus formalized as partially observable Markov decision processes (POMDPs) in which the agent still receives a scalar reward but must contend with uncertainty in its perception~\citep{kaelbling1998planning, pineau2003point, liu2022partially}. These two threads  converge in partially observable multi-agent RL (similarly decentralized POMDPs), where multiple agents, each with only partial observations of the environment, individually or collectively receive reward signals and must account for both uncertainty and the presence of other learning agents~\citep{bernstein2002complexity, oliehoek2016concise, omidshafiei2017deep, zhang2018fully}.

From a reinforcement learning perspective, each entity in the intelligent field may be viewed as an agent, which makes the setting similar to a partially observable multi-agent RL problem. The objective signals exchanged among entities align with the notion of rewards. However, in the intelligent field framework, the notions of ``agent'' and ``environment'' are relative and symmetric: from the standpoint of any given agent, the rest of the field (i.e., all other agents) constitutes the environment. Consequently, the system is viewed as a network of interacting agents whose behaviors and objective signals collectively shape one another’s learning processes. Moreover, from this perspective, the reward mechanism is quite complex, as each agent both generates rewards for others and receives rewards from them.  There is another important concept in the intelligent field framework that is missing from the RL perspective, i.e., decomposability. In the intelligent field framework, each entity can be composed of smaller and simpler entities, and the complex behavior of a larger entity emerges from the composition of these smaller entities. It thus prompts a unified view of reinforcement learning and neural networks, where an agent is composed of many simpler neuron-agents, each exhibiting basic behaviors. The collective dynamics of these simple neuron-agents give rise to the complex behavior of the larger agent. 

A similar phenomenon, where complex collective behavior emerges from simple local dynamic rules, is also observed in various examples of complex systems.

\paragraph{Complex Systems.}
Complex systems composed of interacting, dynamic components are ubiquitous in nature. They include examples as diverse as computational models, physical systems, and neural architectures, usually with emergent behaviors that arise from interactions among simpler components~\citep{ottino2003complex, ladyman2013complex}. 
Cellular automata (CA)~\citep{von1966theory}, a theoretical model for self-replicating cell grids, are discrete dynamical systems in which each cell updates its state based on local interaction rules. It is shown to be Turing complete~\citep{rendell2002turing, cook2004universality}, i.e., given appropriate initial conditions, it can simulate any computable process.
The Ising model is another example illustrating how simple local rules (spin alignments) can yield global phase transitions and rich complexity, bridging conceptually to Hopfield networks and Boltzmann machines as discussed previously. Likewise, random processes on graphs, e.g., epidemic-like spreading, capture how network topology drives complex phenomena in various domains~\citep{easley2010networks}.  

The objective-driven dynamical stochastic field, when the ground space is a discrete graph, can be viewed as another complex system in which local node interactions govern overall behavior. For example, from a cellular automata perspective, the intelligent field can be viewed as a generalized framework that extends traditional CA principles by integrating continuous-time stochastic dynamics and self-adaptive learning rules. While it preserves the local-update rule of cellular automata, it departs from static rule sets by introducing stochastic objective signals exchanged among neighboring entities. Each entity then updates its internal configuration to minimize the objective values provided by others, guiding the system as a whole toward intended goals. 

Complex systems are notoriously difficult to model, creating challenges in designing local update rules that yield a desired global outcome. In order to model and analyze the stochastic dynamical field, as implied by our first two principles, we draw inspiration from the theory of quantum fields, despite no actual quantum physics being involved.

\paragraph{Quantum Fields.}
Our theory begins by introducing a Hilbert space to encode the distributional information, i.e., the ``superposition,'' of possible system configurations.  The dynamics within this Hilbert space are governed by a first-order differential equation, $\frac{\diff}{\dt} \sfield=\bG \sfield$, which is a reformulation of the Kolmogorov forward equation. This resembles the Schrödinger's equation, where the generator of the system dynamics is the Hamiltonian. Moving to a field formulation, we prove that the system's generator $\bG=\sum_{x} \bG(x)$ is the sum of local generators, where the commutator $[\bG(x), \bG(x')]=0$ for non-neighboring points $x$ and $x'$. This mirrors how the total Hamiltonian of a quantum field is obtained as the spatial integral of the Hamiltonian density.  Likewise, we can derive the path integral formalism: $\bra{\field'} e^{\bG T} \ket{\field}=\int \gD \field \ e^{-S[\field]}$, where $S[\field]=\sum_x \int_{0}^{T} \dt \ L(w_{t, x}, w_{t, x}^+)$ is analogous to the action of the path, i.e., the spatial-temporal integral of the Lagrangian density. As we take the limit $H_w\to 0$, i.e., letting the entropy of the system go to zero, the path that minimizes $S[\field]$ dominates the integral, similar to how taking \(\hbar \to 0\) in Feynman’s path integral recovers the classical limit in quantum mechanics.

These theoretical tools, inspired by quantum fields, provide a compact model of the dynamical stochastic field, as described in Section~\ref{sec:field}. From an artificial intelligence perspective, however, it is essential for the model to demonstrate adaptive behavior that is capable of learning to fulfill a specific purpose. Therefore, by further introducing the concept of purposefulness into this system in Section~\ref{sec:IF}, we arrive at the final result of this paper: a theoretical framework for objective-driven dynamical stochastic fields, which we refer to as intelligent fields.

We would like to highlight that the connections discussed above serve primarily as illustrative examples to help clarify how intelligent fields relate to various established perspectives. These examples are intended to offer intuitive insights and suggest possible applications. Nevertheless, as mentioned earlier, each of these perspectives differs in certain respects. We recognize that these examples are not comprehensive and that there can be additional connections beyond the current scope. A deeper investigation into these specific links remains an important direction for future work.

\section{Conclusion}\label{sec:conclusion}

In this work, we introduce a theoretical framework for understanding and analyzing objective-driven dynamical stochastic fields, which we refer to as intelligent fields. By formalizing the principles of complete configuration, locality, and purposefulness, we establish a mathematical foundation for understanding such intricate systems. In addition, we explore design methodologies from the perspective of artificial intelligence.  While the proposed principles and framework provide a foundation for understanding and designing such systems, there remain significant opportunities for further exploration and refinement, ranging from theoretical advancements to practical applications.

%% file: sections/appendix.tex
\newpage
\appendix 
\section*{\centering\LARGE Appendix}

\section*{Contents}
\begin{itemize}
    \item Section \ref{sec:proof-field}: Proofs in Section~\ref{sec:proof-field}
    \begin{itemize}
        \item \ref{proof-prop:dynamics-generator}: Proof of Proposition~\ref{prop:dynamics-generator}
        \item \ref{proof-fact:iso-tensor}: Proof of Fact \ref{fact:iso-tensor}
        \item \ref{proof-prop:decomposition}: Proof of Proposition~\ref{prop:decomposition}
        \item \ref{proof-thm:decomposition-field}: Proof of Theorem~\ref{thm:decomposition-field}
        \item \ref{proof-prop:commute}: Proof of Proposition~\ref{prop:commute}
        \item \ref{proof-prop:decomp-Lagrangian}: Proof of Proposition~\ref{prop:decomp-Lagrangian}
        \item \ref{proof-thm:path-int}: Proof of Theorem~\ref{thm:path-int}
        \item \ref{proof-prop:lift-operators-a-case}: Proof of Proposition~\ref{prop:lift-operators-a-case}
        \item \ref{proof-prop:min-L}: Proof of Proposition~\ref{prop:min-L}
    \end{itemize}
    \item Section \ref{sec:proof-IF}: Proofs in Section~\ref{sec:IF}
    \begin{itemize}
        \item \ref{proof-prop:stationary}: Proof of Proposition~\ref{prop:stationary}
        \item \ref{proof-cor:avg-valid}: Proof of Corollary~\ref{cor:avg-valid}
        \item \ref{proof-thm:ID-case}: Proof of Theorem~\ref{thm:ID-case}
        \item \ref{proof-fact:subH}: Proof of Fact~\ref{fact:subH}
        \item \ref{proof-fact:proj-G}: Proof of Fact~\ref{fact:proj-G}
        \item \ref{proof-prop:bound-S}: Proof of Proposition~\ref{prop:bound-S}
        \item \ref{proof-prop:grad}: Proof of Proposition~\ref{prop:grad}
        \item \ref{proof-fact:lifted-decompose}: Proof of Fact~\ref{fact:lifted-decompose}
        \item \ref{proof-prop:grad-entity}: Proof of Proposition~\ref{prop:grad-entity}
        \item \ref{proof-thm:grad-field}: Proof of Theorem~\ref{thm:grad-field}
        \item \ref{proof-prop:local-ppsv}: Proof of Proposition~\ref{prop:local-ppsv}
        \item \ref{proof-lemma:SPiG}: Proof of Lemma~\ref{lemma:SPiG}
        \item \ref{proof-prop:P-properties}: Proof of Proposition~\ref{prop:P-properties}
        \item \ref{proof-prop:PQ-properties}: Proof of Proposition~\ref{prop:PQ-properties}
        \item \ref{proof-thm:ppsq-grad}: Proof of Theorem~\ref{thm:ppsq-grad}
    \end{itemize}
\end{itemize}

\section{Proofs in Section~\ref{sec:field}}\label{sec:proof-field}

\subsection{Proposition~\ref{prop:dynamics-generator}}\label{proof-prop:dynamics-generator}
\textbf{Proposition~\ref{prop:dynamics-generator} (Dynamics of the System).}
		The time evolution of the system satisfies the following first-order differential equation,
	\begin{align}
		\frac{\diff }{\dt}\ket{\varphi(t)}=\bG \ket{\varphi(t)},
	\end{align}
	and its solution is 
	\begin{align}
		\ket{\varphi(t)}=e^{\bG t}\ket{\varphi(0)}.
	\end{align}
	Note that the exponential of the linear operator is defined as $
	e^{\bG t}=1+\sum_{n=1}^{\infty} \frac{1}{n!}(\bG t)^n$. 	When it is clear in the context,  we  may simply use $\sfield$ to denote $\ket{\varphi(t)}$ for notational ease.

\begin{proof}
	 Consider a process $\field(t)$ where the probability distribution of its configuration  at time $t$ is characterized by $\ket{\varphi(t)}$. We may omit the time label and simply use $\sfield$ to denote $\ket{\varphi(t)}$.
     To verify if any two elements in $\gH$ are equal, we may verify that they are equal in every basis. Thus, we can prove the proposition by verifying that $\frac{\diff }{\dt}\sfield$ and $\bG \sfield$ are equal in every basis. 
	\begin{align}
		\bra{\field'} \frac{\diff }{\dt}\sfield &=\lim_{\Dt\to 0} \frac{1}{\Dt} \bra{\field'}  (\ket{\varphi(t+\Dt)}-\ket{\varphi(t)}) \\
		&=\lim_{\Dt\to 0}\frac{1}{\Dt}\left( \braket{\field'}{\varphi(t+\Dt)}-  \braket{\field'}{\varphi(t)}\right)\\
		&=\lim_{\Dt\to 0}\frac{1}{\Dt}\left( p(\field(t+\Dt)=\field')-p(\field(t)=\field')\right).	 \label{eq:dynamic-2}
	\end{align}
	To proceed, we may apply the law of total probability on $ p(\field(t+\Dt)=\field')$ conditioning on $\field(t)=\field$. Moreover, write $\ket{\varphi}=\sum_{\field}\varphi^\field \ket{\field}$ and we have $p(\field(t)=\field)=\sfcoef^{\field}$. Thus, 
	\begin{align}
		(\ref{eq:dynamic-2})&=\lim_{\Dt\to 0}\frac{1}{\Dt}\Big( \sum_{\field} p(\field(t+\Dt)=\field'\mid \field(t)=\field) \cdot p(\field(t)=\field)  - p(\field(t)=\field')\Big)\\
		&= \lim_{\Dt\to 0}\frac{1}{\Dt} \sum_{\field} \left(  p(\field(t+\Dt)=\field')\mid \field(t)=\field)   - \delta_{\field}^{\field'}\right)\sfcoef^{\field}\\
		&=  \sum_{\field} \lim_{\Dt\to 0}\frac{p(\field(t+\Dt)=\field'\mid \field(t)=\field)   - \delta_{\field}^{\field'}}{\Dt} \cdot \sfcoef^{\field}\\
        &=  \sum_{\field} \lim_{\Dt\to 0}\frac{p_\Dt(\field'\mid \field)   - \delta_{\field}^{\field'}}{\Dt} \cdot \sfcoef^{\field}\\
		&=\sum_\field \bra{\field'} \bG\ket{\field}\  \sfcoef^{\field}=\bra{\field'} \bG\left(\sum_{\field}\sfcoef^{\field}\ket{\field}\right)\  \\
		&= \bra{\field'} \bG \sfield .
	\end{align}
    Since $\field'$ is arbitrary, we must have $\frac{\diff }{\dt}\sfield=\bG \sfield$. 
	
	This is a standard first order differential equation, where the solution is known to be $\ket{\varphi(t)}=e^{\bG t} \ket{\varphi(0)}$. 
	We may verify that
	\begin{align}
		\frac{\diff}{\dt} \ket{\varphi(t)} =\frac{\diff}{\dt} e^{\bG t} \ket{\varphi(0)} =  \bG e^{\bG t} \ket{\varphi(0)} =\bG\ket{\varphi(t)}.
	\end{align}
	
\end{proof}

\subsection{Fact \ref{fact:iso-tensor}}\label{proof-fact:iso-tensor}

\textbf{Fact~\ref{fact:iso-tensor}.} Given two sets $\gset,\gset'$, the Hilbert space $\gH(\gset\times \gset')$ and the tensor product $\gH(\gset )\otimes \gH(\gset')$ are the ``same''. Formally, they are isometrically isomorphic:
\begin{align}
    \gH(\gset\times \gset') \cong \gH(\gset )\otimes \gH(\gset'),
\end{align}

\begin{proof}
We can construct a linear map $\rmF:\gH(\gset\times \gset') \to \gH(\gset )\otimes \gH(\gset') $ as $\rmF \ket{\field\field'}=\ket{\field}\otimes \ket{\field'}$. 
By definitions one can verify that it is bijective, and it preserves the inner product. 
\end{proof}

\subsection{Proposition~\ref{prop:decomposition}}\label{proof-prop:decomposition}
\textbf{Proposition~\ref{prop:decomposition} (Decomposition of the Infinitesimal Generator $\bG$).}
The locality principle in the two-entity view (\eqref{eq:3-1-1}) implies
\begin{align}
    \bG= \bM +  \bN.
\end{align}
\begin{proof}
	Before we proceed, we are going to use the following properties. By definition of $\bG, \bM, \bN$ (Definition~\ref{def:inf-generator}\&\ref{def:generator-entity}), we know 
	\begin{align}
		 \bra{\alpha' \beta' \mu' \nu'}\bG\ket{\alpha \beta \mu \nu} \cdot \Dt &=  p_{\Dt}(\alpha'\beta' \mu'\nu'\mid \alpha \beta\mu\nu)-\delta_{\alpha \beta \mu \nu}^{\alpha' \beta' \mu' \nu'}+o(\Dt),\\
		\bra{\alpha' \mu' }\bM \ket{\alpha \beta \mu }\cdot \Dt&=  p_{ \Dt}(\alpha' \mu' \mid \alpha \beta \mu)-\delta_{\alpha  \mu }^{\alpha' \mu' }+o(\Dt),\\
		 \bra{\beta' \nu' } \bN \ket{\alpha \beta \nu } \cdot \Dt&= p_{ \Dt}(\beta' \nu' \mid \alpha \beta \nu)-\delta_{\beta\nu}^{\beta'\nu'} + o(\Dt).
	\end{align}
	Moreover, we will need the independence relationship of the two entities (\eqref{eq:3-1-1}), i.e.,   
	\begin{align}
		p_{\Dt}(\alpha'\beta' \mu'\nu'\mid \alpha \beta\mu\nu)=p_{ \Dt}(\alpha' \mu' \mid \alpha \beta \mu)\cdot p_{ \Dt}(\beta' \nu' \mid \alpha \beta \nu)+o(\Dt).
	\end{align}
	
	We may prove $\bG=  \bM +   \bN$ by comparing 
	\begin{align}
		 \bra{\alpha' \beta' \mu' \nu'}\bG\ket{\alpha \beta \mu \nu} \quad \text{and}\quad  \bra{\alpha' \beta' \mu' \nu'} (  \bM +   \bN)\ket{\alpha \beta \mu \nu}
	\end{align}
	for  all $ \alpha \beta \mu \nu, \alpha' \beta' \mu' \nu'\in \Omega$. 
	
	We start by definition that
	\begin{align}
		 \bra{\alpha' \beta'\mu' \nu'} \bG\ket{\alpha \beta \mu \nu}\cdot \Dt &=  p_{\Dt}(\alpha'\beta'\mu'\nu'\mid \alpha \beta\mu\nu)-\delta_{\alpha \beta \mu \nu}^{\alpha' \beta' \mu' \nu'}+o(\Dt)\\
		&= p_{ \Dt}(\alpha' \mu' \mid \alpha \beta \mu)\cdot p_{ \Dt}(\beta' \nu' \mid \alpha \beta \nu) -\delta_{\alpha \beta \mu \nu}^{\alpha' \beta' \mu' \nu'} +o(\Dt). \label{eq:decom-1}
	\end{align}
	Next, we convert $p_{ \Dt}$  back to the expressions with $\bM, \bN$. 
	\begin{align}
		(\ref{eq:decom-1})&=(p_{ \Dt}(\alpha' \mu' \mid \alpha \beta \mu)-\delta_{\alpha  \mu }^{\alpha' \mu' }+\delta_{\alpha  \mu }^{\alpha' \mu' } )\cdot (p_{ \Dt}(\beta' \nu' \mid \alpha \beta \nu) - \delta_{\beta\nu}^{\beta'\nu'} +\delta_{\beta\nu}^{\beta'\nu'} ) -\delta_{\alpha \beta \mu \nu}^{\alpha' \beta' \mu' \nu'} +o(\Dt)\\
		&=(\bra{\alpha' \mu' }\bM \ket{\alpha \beta \mu }\cdot \Dt+\delta_{\alpha  \mu }^{\alpha' \mu' } )\cdot ( \bra{\beta' \nu' } \bN \ket{\alpha \beta \nu } \cdot \Dt +\delta_{\beta\nu}^{\beta'\nu'} ) -\delta_{\alpha \beta \mu \nu}^{\alpha' \beta' \mu' \nu'} +o(\Dt)\\
		&=\bra{\alpha' \mu' }\bM \ket{\alpha \beta \mu }\cdot \Dt \cdot \delta_{\beta\nu}^{\beta'\nu'} +   \bra{\beta' \nu' } \bN \ket{\alpha \beta \nu } \cdot \Dt \cdot \delta_{\alpha\mu}^{\alpha'\mu'} +o(\Dt). \label{eq:decom-2}
	\end{align}
	Note that in the above we implicitly put all $O((\Dt)^2)$ terms into $o(\Dt)$, and that $\delta_{\alpha  \mu }^{\alpha' \mu' } \delta_{\beta\nu}^{\beta'\nu'}=\delta_{\alpha \beta \mu \nu}^{\alpha' \beta' \mu' \nu'}$.
	
	It remains to convert the above to the expression with their embeddings $  \bM,   \bN$.  By definition of the embeddings (Definition~\ref{def:generator-entity}) we obtain
	\begin{align}
		(\ref{eq:decom-2})&= \bra{\alpha'\mu'}  \bM\ket{\alpha\beta\mu}\cdot  \braket{\beta'\nu'}{\beta\nu}  \cdot \Dt  +  \bra{\beta' \nu' } \bN \ket{\alpha \beta \nu } \cdot \braket{\alpha'\mu'}{\alpha\mu} \cdot \Dt + o(\Dt)\\
		&=\bra{\alpha' \beta' \mu' \nu'} ((\bM\ket{\alpha\beta\mu}) \otimes \ket{\beta\nu})\cdot \Dt + \bra{\alpha' \beta' \mu' \nu'} ((\bN\ket{\alpha\beta\nu}) \otimes \ket{\alpha\mu})\cdot \Dt +o(\Dt)\\
		&=\bra{\alpha' \beta' \mu' \nu'} (  \bM +   \bN)\ket{\alpha \beta \mu 	\nu}\cdot \Dt+o(\Dt).
	\end{align}
	Therefore, 
	\begin{align}
		 \bra{\alpha' \beta'\mu' \nu'} \bG\ket{\alpha \beta \mu \nu}\cdot \Dt=\bra{\alpha' \beta' \mu' \nu'} (  \bM +   \bN)\ket{\alpha \beta \mu 	\nu}\cdot \Dt+o(\Dt).
	\end{align}
	Dividing both sides by $\Dt$ and sending $\Dt\to 0$ proves
	\begin{align}
		 \bra{\alpha' \beta' \mu' \nu'}\bG\ket{\alpha \beta \mu \nu} =   \bra{\alpha' \beta' \mu' \nu'} (  \bM +   \bN)\ket{\alpha \beta \mu \nu}.
	\end{align}
	Since $ \alpha \beta \mu \nu, \alpha' \beta' \mu' \nu'\in \Omega$ are arbitrary, we conclude that 
	\begin{align}
		\bG=  \bM+  \bN.
	\end{align}

\end{proof}

\subsection{Theorem~\ref{thm:decomposition-field}}\label{proof-thm:decomposition-field}

\textbf{Theorem~\ref{thm:decomposition-field} (Decomposition of the Infinitesimal Generator $\bG$).}
	Locality (Principle~\ref{pinciple:locality-detail}) implies
	\begin{align}
		\bG= \sum_x \bG(x).
	\end{align}

\begin{proof}
	We can prove this theorem very similar to what we do for Proposition~\ref{prop:decomposition}.
	By definition of the generators $\bG$ and $\bG(x)$ (Definition~\ref{def:inf-generator}\&\ref{def:generator-field}),  we have
	\begin{align}
		\bra{\field'} \bG\ket{\field}\cdot \Dt&=  p_{\Dt}(\field'\mid \field)-\delta_{\field}^{\field'} + o(\Dt),\\
		 \bra{\field'_{|\gset_x} }\bG(x) \ket{\field_{|\nbhd_x} }\cdot\Dt&=  p_{ \Dt}(\field'_{|\gset_x} \mid \field_{|\nbhd_x} )-\delta_{\field_{|\gset_x}}^{\field'_{|\gset_x}}  +o(\Dt).
	\end{align}
	Moreover, locality (Principle~\ref{pinciple:locality-detail}) implies  
	\begin{align}
		p_\Dt(\field'\mid\field)=\prod_x p_\Dt(\field'_{|\gset_x}\mid\field_{|\nbhd_x}) + o(\Dt).
	\end{align}
	
	We may prove $\bG=  \sum_x \bG(x)$ by comparing 
	\begin{align}
		\bra{\field'} \bG\ket{\field} \quad \text{and}\quad \bra{\field'} \left( \sum_x \bG(x) \right)\ket{\field}
	\end{align}
	for  all $ \field, \field'\in \Omega$. 
	
	We start by definition that
	\begin{align}
		\bra{\field'} \bG\ket{\field}\cdot \Dt &=  p_{\Dt}(\field'\mid \field)-\delta_{\field}^{\field'}+o(\Dt)= \prod_x p_\Dt(\field'_{|\gset_x}\mid\field_{|\nbhd_x}) -\delta_{\field}^{\field'} +o(\Dt). \label{eq:decom-field-1}
	\end{align}
	Next, we convert $p_{ \Dt}$  back to the expressions with $\bG(x)$. 
	\begin{align}
		(\ref{eq:decom-field-1})
		&=\prod_x \left(p_\Dt(\field'_{|\gset_x}\mid\field_{|\nbhd_x}) - \delta_{\field_{|\gset_x}}^{\field'_{|\gset_x} }+ \delta_{\field_{|\gset_x}}^{\field'_{|\gset_x} } \right) -\delta_{\field}^{\field'} +o(\Dt)\\		
		&=\prod_x \left( \bra{\field'_{|\gset_x} }\bG(x) \ket{\field_{|\nbhd_x} } \cdot \Dt+ \delta_{\field_{|\gset_x}}^{\field'_{|\gset_x} } \right) -\delta_{\field}^{\field'} +o(\Dt)\\		
		&=\sum_x  \bra{\field'_{|\gset_x} }\bG(x) \ket{\field_{|\nbhd_x} } \cdot \Dt \cdot \delta_{\field_{|\gset/\gset_x}}^{\field'_{|\gset/\gset_x} }  +o(\Dt). \label{eq:decom-field-2}
	\end{align}
	Note that in the above we implicitly put all higher order terms into $o(\Dt)$, and that $\prod_x \delta_{\field_{|\gset_x}  }^{\field'_{|\gset_x} } = \delta_{\field}^{\field'}$.
	
	It remains to convert the above to the expression with their embeddings. By definition of the embeddings (Definition~\ref{def:generator-field}) we obtain
	\begin{align}
		(\ref{eq:decom-field-2})
		&=\sum_x  \bra{\field'_{|\gset_x} }\bG(x) \ket{\field_{|\nbhd_x} }  \cdot \braket{\field'_{|\gset/\gset_x}}{\field_{|\gset/\gset_x}}  \cdot \Dt +o(\Dt)\\
		&=\sum_x \left(\bra{\field'_{|\gset_x} }\otimes \bra{\field'_{|\gset/\gset_x}}\right) (\bG(x) \ket{\field_{|\nbhd_x} })\otimes \ket{\field_{|\gset/\gset_x}}   \cdot \Dt +o(\Dt)\\
		&=\sum_x  \bra{\field' } \bG(x) \ket{\field}   \cdot \Dt +o(\Dt).
	\end{align}
	Therefore, 
	\begin{align}
		\bra{\field'} \bG\ket{\field}\cdot \Dt=\bra{\field'} \left( \sum_x \bG(x)\right)\ket{\field}\cdot \Dt+o(\Dt).
	\end{align}
	Dividing both sides by $\Dt$ and sending $\Dt\to 0$ proves
	\begin{align}
		\bra{\field'} \bG\ket{\field} =   \bra{\field'} \left(\sum_xG(x)\right)\ket{\field}.
	\end{align}
	Since $ \field, \field'\in \Omega$ are arbitrary, we conclude that 
	\begin{align}
		\bG=  \sum_x \bG(x).
	\end{align} 
	
\end{proof}

\subsection{Proposition~\ref{prop:commute}}\label{proof-prop:commute}

\textbf{Proposition~\ref{prop:commute} (Commutation Relations of the Local Generators).}
	If $x, x'$ are not neighbors, i.e., $x \nsim x'$, their generators commute. Formally,  
	\begin{align}
		[\bG(x), \bG(x')]=0, \qquad \text{if } x\nsim x',
	\end{align}
	where the commutator $[\bG(x), \bG(x')]=\bG(x)\bG(x')-\bG(x')\bG(x)$.

\begin{proof}
	Since $x\nsim x'$, we know that the local neighborhoods $\nbhd_x$ and $\nbhd_{x'}$ correspond to disjoint sections in $\gset$. Therefore, noting that tensor products are automatically ordered, we have
	\begin{align}
		\bG(x')\bG(x)\ket{\field}&=\bG(x') \left((\bG(x)\ket{\field_{|\nbhd_x}})\otimes \ket{\field_{|\gset/\gset_x}} \right)\\
		&=(\bG(x)\ket{\field_{|\nbhd_x}})\otimes  (\bG(x')\ket{\field_{|\nbhd_{x'}}}) \otimes \ket{\field_{|\gset/(\gset_x\times \gset_{x'})}} \\
		&=\bG(x)\bG(x')\ket{\field}.
	\end{align}
\end{proof}

\subsection{Proposition~\ref{prop:decomp-Lagrangian}}\label{proof-prop:decomp-Lagrangian}

\textbf{Proposition~\ref{prop:decomp-Lagrangian} (Decomposition of the Lagrangian).}
\begin{align}
    L(w_t, w_t^+)=\sum_x L(w_{t, x}, w_{t, x}^+).
\end{align}    
\begin{proof}
Applying Theorem~\ref{thm:decomposition-field} to the system Lagrangian $L(t, \field, \field^+)$, we have
\begin{align}
	L(w_t, w_t^+)=-\left(\sum_x G(x)^{\field_t}_{\field_t}\right)-\delta({\field_t^+\neq \field_t})\cdot \log \left( \sum_x G(x)^{\field_{t}^+}_{\field_{t}}\right).
\end{align}
Note that, by Theorem~\ref{thm:decomposition-field}, a path $\field_t$ can only jump at one $x$ each time. By definition of the embedding (Definition~\ref{def:generator-field},) $G(x)^{\field_{t}^+}_{\field_{t}}$ can be non-zero only if $\field_{t}^+$ and $\field_{t}$ agree on $\gset/\gset_x$. Therefore, if there is $\field_t\to \field_t^+$ happens at $x$, then all $x'\neq x$ would have $G(x')^{\field_{t}^+}_{\field_{t}}=0$.

Therefore, we have 
\begin{align}
	\delta({\field_t^+\neq \field_t})\cdot \log \left( \sum_x G(x)^{\field_{t}^+}_{\field_{t}}\right)&=\sum_x \delta({\field_t^+}_{|\gset_x}\neq {\field_t}_{|\gset_x})\cdot \log \left( G(x)^{{\field_t^+}_{|\gset_x}}_{{\field_t}_{|\nbhd_x}}\right).
\end{align}

Therefore, by definition 
\begin{align}
	L(w_t, w_t^+)=\sum_x L(w_{t, x}, w_{t, x}^+).
\end{align}
\end{proof}

\subsection{Theorem~\ref{thm:path-int}}\label{proof-thm:path-int}
\textbf{Theorem~\ref{thm:path-int} (A Path Integral Formalism of the System).}
    The probability that the system evolves from $\field$ to $\field'$ after time $T$ can be expressed in the following ways.
    \begin{align}
	p_T(\field'\mid\field)=\bra{\field'}e^{\bG T} \ket{\field}  =\int \gD \field \ \exp\left\{-\sum_x\int_{0}^{T} \dt \ L(w_{t, x}, w_{t, x}^+)\right\}, %
\end{align}
where the integral is done over paths starting from $\field$ to $\field'$ in time $T$.

\begin{proof}
    We begin by classifying different paths by the number of their jumps, i.e., the number of configuration changes. Denote $\field_{0:T, \Dt, k}$ as a path having  $k$ jumps, and denote $t_n=n\Dt$, we have
\begin{align}
	p_t(\field'\mid\field)&=\sum_{k=0}^N\sum_{\field_{0:T, \Dt, k}}^{\field\to \field'} \prod_{n=1}^N \bra{\field_{t_n}} e^{\bG\Dt} \ket{\field_{t_{n-1}}}\\
	&=\sum_{k=0}^N  \sum_{\field_{0:T, \Dt, k}}^{\field\to \field'} \cdot (\Dt)^k  \cdot \frac{1}{(\Dt)^k} \prod_{n=1}^N \bra{\field_{t_n}} e^{\bG\Dt} \ket{\field_{t_{n-1}}}\\ 
	&=\sum_{k=0}^N  \sum_{\field_{0:T, \Dt, k}}^{\field\to \field'} \cdot (\Dt)^k  \cdot \exp\left\{ -k\log\left(\Dt \right) + \sum_{n=1}^N \log \ \bra{\field_{t_n}} e^{\bG\Dt} \ket{\field_{t_{n-1}}} \right\} \label{eq:3-3-1}
\end{align}
Recall that we want $\Dt\to 0$ and $N\to \infty$. To make sense of this, we view the above as three parts. The first part is a summation over paths, and the second part is $(\Dt)^k $ which is viewed as a path measure. Thus, as $\Dt\to 0$ and $N\to\infty$, the first two parts together become essentially an integral over all the paths going from $\field$ at time $t=0$ and ending in $\field'$ at time $t=T$. It remains to determine the last term, and specifically its exponent:
\begin{align}
	 -k\log\left(\Dt \right) + \sum_{n=1}^N \log \ \bra{\field_{t_n}} e^{\bG\Dt} \ket{\field_{t_{n-1}}}, \label{eq:3-3-2} 
\end{align}
for which we will need the following fact: 
\begin{align}
	e^{\bG\Dt}=1+\bG\Dt+o(\Dt) \quad \text{and}\quad \bG=\ket{\field'} G^{\field'}_{\field} \bra{\field}.
\end{align}
There are two cases. The first case is that $\field_{t_n}=\field_{t_{n-1}}$, and we can see that
\begin{align}
	\log \ \bra{\field_{t_n}} e^{\bG\Dt} \ket{\field_{t_{n-1}}} = \log\left(1+G^{\field_{t_{n}}}_{\field_{t_{n-1}}}\cdot \Dt + o(\Dt) \right)=G^{\field_{t_{n}}}_{\field_{t_{n-1}}}\cdot \Dt + o(\Dt).
\end{align}
For the second case, i.e., $\field_{t_n}\neq\field_{t_{n-1}}$,  there are exactly $k$ jumps for the given path. Therefore, we distribute the $-k\log(\Dt)$ into this case and obtain
\begin{align}
	- \log(\Dt)+\log \ \bra{\field_{t_n}} e^{\bG\Dt} \ket{\field_{t_{n-1}}} &=\log \left( \bra{\field_{t_n}} e^{\bG\Dt} \ket{\field_{t_{n-1}}}/\Dt \right)\\
	&= \log \left( G^{\field_{t_n}}_{\field_{t_{n-1}}} + o(\Dt)\right).
\end{align}
Therefore, putting the two cases together, we finally have
\begin{align}
	(\ref{eq:3-3-2})&=\sum_{n=1}^N \delta_{\field_{t_{n-1}}}^{\field_{t_{n}}} G^{\field_{t_{n}}}_{\field_{t_{n-1}}} \cdot \Dt + (1-\delta_{\field_{t_{n-1}}}^{\field_{t_{n}}})\log \left( G^{\field_{t_n}}_{\field_{t_{n-1}}} \right) + o(\Dt)\\
	&=\sum_{n=1}^N \Dt \left( \delta_{\field_{t_{n-1}}}^{\field_{t_{n}}} G^{\field_{t_{n}}}_{\field_{t_{n-1}}} + (1-\delta_{\field_{t_{n-1}}}^{\field_{t_{n}}})\log \left( G^{\field_{t_n}}_{\field_{t_{n-1}}} \right)/\Dt + o(\Dt)/\Dt\right) \label{eq:3-3-3}.
\end{align}
When taking $\Dt\to 0$, or equivalently $N=T/\Dt\to \infty$, the $\sum_{n=1}^N \Dt$ above becomes an integral $\int_0^T \dt$. The last term becomes $o(\Dt)/\Dt\to 0$. Then, $\field_t$ is discontinuous w.r.t. time but with only finite discontinuous points, i.e., jumps between configurations.  Since there are only finite jumps, the first term becomes just $G^{\field_t}_{\field_t}$, i.e.,
\begin{align}
    \lim_{\Dt\to 0 } \sum_{n=1}^N \Dt \left( \delta_{\field_{t_{n-1}}}^{\field_{t_{n}}} G^{\field_{t_{n}}}_{\field_{t_{n-1}}}\right) = \int_0^T\dt \ G^{\field_t}_{\field_t}.
\end{align}

Then, the second term can be formulated by the unit impulse function (Definition~\ref{def:unit-impulse}).
\begin{align}
	\lim_{\Dt\to 0 }\sum_{n=1}^N \Dt \cdot  (1-\delta_{\field_{t_{n-1}}}^{\field_{t_{n}}})\log \left( G^{\field_{t_n}}_{\field_{t_{n-1}}} \right)/\Dt=\int_0^T\dt \ \delta({\field_t^+\neq \field_t})\cdot \log \left( G^{\field_{t}^+}_{\field_{t}} \right).
\end{align}
Therefore, \eqref{eq:3-3-3} can be formulated as the integral of the Lagrangian as defined in Definition~\ref{def:lagrangian}.
Thus, we can express \eqref{eq:3-3-2} in a very compact form:
\begin{align}
	\lim_{\Dt\to 0}  -k\log\left(\Dt \right) + \sum_{n=1}^N \log \ \bra{\field_{t_n}} e^{\bG\Dt} \ket{\field_{t_{n-1}}}=-\int_{0}^{T} \dt \ L(w_t, w_t^+).
\end{align}
Thus, adopting the path integral notation, we find the path integral formulation of the propagation probability 
\begin{align}
	p_T(\field'\mid\field)&=\lim_{\Dt\to 0} \sum_{k=0}^N  \sum_{\field_{0:T, \Dt, k}}^{\field\to \field'} \cdot (\Dt)^k  \cdot \exp\left\{ -k\log\left(\Dt \right) + \sum_{n=1}^N \log \ \bra{\field_{t_n}} e^{\bG\Dt} \ket{\field_{t_{n-1}}} \right\}\\
	&=:\int \gD \field \ \exp\left\{-\int_{0}^{T} \dt \ L(w_t, w_t^+)\right\}, \label{eq:3-3-4}
\end{align}
where $\int \gD \field $  is integrating over all paths $\field_t$ starting at $\field_0=\field$ and ending at $\field_T=\field'$.

Decomposing the system Lagrangian to the sum of field Lagrangian over $x\in \gX$, according to Proposition~\ref{prop:decomp-Lagrangian}, we arrive at the final expression.
\begin{align}
	p_T(\field'\mid\field) =  \bra{\field'} e^{\bG T} \ket{\field}=\int \gD \field \ \exp\left\{-\sum_x\int_{0}^{T} \dt \ L(w_{t, x}, w_{t, x}^+)\right\}.
\end{align}

\end{proof}

\subsection{Proposition~\ref{prop:lift-operators-a-case}}\label{proof-prop:lift-operators-a-case}
\textbf{Proposition~\ref{prop:lift-operators-a-case}.}
Consider a function in the form of 
	\begin{align}
		\gamma(t, \field, \field^+)=\delta({\field_t^+\neq \field_t})\cdot \gamma_{\field_t^+\field_t},
	\end{align}
	It corresponds to an operator $\pps:\lH\to \sR$ defined as 
	\begin{align}
		\pps := \sum_{\field,\field'\in \gset}\gamma_{\field'\field} \bra{\field'\field} ,\qquad \text{where }\ \forall \field: \gamma_{\field\field}=0.
	\end{align} 
	Consequently, we have:
	\begin{align}
		\E\left[\int_0^T  \dt\ \gamma(t, \field, \field^+) \Bigm| \field_0=\field\right]=\int_{0}^{T} \dt\  \pps \lG  e^{\bG t}\ket{\field} .
	\end{align}

\begin{proof}
	We begin by taking an infinitesimal step $\Dt$, i.e.,
	\begin{align}
		\E\left[\int_0^\Dt \dt\ \gamma(t, \field, \field^+) \Bigm| \field_0=\field\right]&=\sum_{\field':\field'\neq \field} p_{\Dt}(\field'\mid \field) \gamma_{\field' \field} +  o(\Dt).  \label{eq:3-4-3}
	\end{align}
	Recall that for $\field'\neq\field$ we know $p_{\Dt}(\field'\mid \field)=G^{\field'}_{\field}\cdot\Dt+o(\Dt)$. Plugging this in, we have
	\begin{align}
		(\ref{eq:3-4-3})&=  \Dt\cdot \sum_{\field':\field'\neq \field}G^{\field'}_{\field} \gamma_{\field'\field} + o(\Dt)=   \sum_{\field'\in \gset}G^{\field'}_{\field} \gamma_{\field'\field} \cdot \Dt+ o(\Dt)\\
		&= \left(\sum_{\field''}\gamma_{\field''\field} \bra{\field''\field}\right)\left(\sum_{\field'} \ket{\field' \field}G^{\field'}_{\field}  \right)\cdot \Dt+ o(\Dt)\\
		&=\pps \lG \ket{\field}\cdot \Dt + o(\Dt). 
	\end{align}
	Therefore, stitching the little steps together and taking $\Dt=T/N$,  we have
	\begin{align}
		\E&\left[\int_0^T  \dt\ \gamma(t, \field, \field^+) \Bigm| \field_0=\field\right]=\sum_{n=1}^N \E\left[\int_{(n-1)\Dt}^{n\Dt} \dt\ \gamma(t, \field, \field^+) \Bigm| \field_0=\field\right]\\
		&=\sum_{n=1}^N \sum_{\field'} p_{(n-1)\Dt}(\field'\mid \field) \E\left[\int_{(n-1)\Dt}^{n\Dt} \dt\ \gamma(t, \field, \field^+) \Bigm| \field_{(n-1)\Dt}=\field'\right]\\
		&=\sum_{n=1}^N \sum_{\field'} \bra{\field'} e^{(n-1)\bG\Dt}\ket{\field} \cdot \left(\pps \lG \ket{\field'}\cdot \Dt + o(\Dt)\right)\\
		&=\sum_{n=1}^N \Dt \cdot \left(\pps \lG  e^{(n-1)\bG\Dt}\ket{\field} + o(\Dt)/\Dt\right)\\
		&=\int_{0}^{T} \dt\  \pps \lG  e^{\bG t}\ket{\field} .\tag{as $\Dt\to 0$ and $N\to \infty$}
	\end{align}

\end{proof}

\subsection{Proposition~\ref{prop:min-L}}\label{proof-prop:min-L}
\textbf{Proposition~\ref{prop:min-L} (Minimizing the Integral of Lagrangian Implies Determinism).}
Given a fixed frequency $\forall \field: |G^{\field}_{\field}|=K$, we can see that $\frac{G^{\field'}_{\field}}{K}$ represents a probability distribution over $\field':\field'\neq \field$. We denote Shannon entropy as
\begin{align}
    H_\field:=-\sum_{\field':\field'\neq \field} \frac{G^{\field'}_{\field}}{K} \log\left( \frac{ G^{\field'}_{\field} }{K} \right).
\end{align}
Then, we have 
\begin{align}
    \tilde{\bm{L}} \lG \ket{\field}=-K\log K+H_\field.
\end{align}
Therefore, the following inequality holds:
\begin{align}
    \E\left[\int_0^T  \dt\ L(\field_t, \field_t^+) \Bigm| \field_0=\field\right]\geq KT(1-\log K),
\end{align}
where the minimum is achieved when the system dynamic is deterministic, i.e., every $\field$ has only one configuration $\field'$ it can jump to.

\begin{proof}
	For any $\ket{\field}\in \gH$, we have 
\begin{align}
	\tilde{\bm{L}} \lG \ket{\field} &=-\sum_{\field':\field'\neq \field} G^{\field'}_{\field}\log G^{\field'}_{\field}\\
	&=-K\sum_{\field':\field'\neq \field} \frac{G^{\field'}_{\field}}{K} \left(\log\left( \frac{ G^{\field'}_{\field} }{K} \right) +\log K \right)\\
	&=-K\log K-K\sum_{\field':\field'\neq \field} \frac{G^{\field'}_{\field}}{K} \log\left( \frac{ G^{\field'}_{\field} }{K} \right)\\
	&=-K\log K+K H_\field\\
	& \geq -K\log K. \tag{because $H_\field\geq 0$}
\end{align}
The inequality is achieved when $H_\field=0$, i.e., there is only one $\field'$ such that $\field$ can jump to.

Therefore, for any normalized state vector $\sfield\in \gH$, we have 
\begin{align}
	\tilde{\bm{L}} \lG \sfield = \sum_\field\tilde{\bm{L}} \lG \sfcoef^\field \ket{\field}\geq \sum_\field \sfcoef^\field\cdot (-K\log K)=-K\log K.
\end{align}
Finally, by Proposition~\ref{prop:lift-operators-a-case}, and that $e^{\bG t}$ preserve the normalization, we can see that 
\begin{align}
	E\left[\int_0^T  \dt\ L(\field_t, \field_t^+) \Bigm| \field_0=\field\right]&=KT+\int_{0}^{T} \dt\  \tilde{\bm{L}} \lG  e^{\bG t}\ket{\field}\\
	&\geq KT(1-\log K).
\end{align}
\end{proof}

\section{Proofs in Section~\ref{sec:IF}}\label{sec:proof-IF}

\subsection{Proposition~\ref{prop:stationary}}\label{proof-prop:stationary}
\textbf{Proposition~\ref{prop:stationary} (The Stationary State Always Exists).}
    Given that the configuration set $\gset$ is finite, the stationary state 
    \begin{align}
        \ket{\varphi(\infty)}:=\lim_{t\to \infty} \ket{\varphi(t)}    
    \end{align}
    always exists.
\begin{proof}
    The proof is not as innocent as its result looks to be. The proof consists of two parts: (1) we use the Gershgorin circle theorem to prove that all eigenvalues of $\bG$ have negative real parts, or they are exactly $0$; (2) we use Jordan normal form to show that $e^{\bG t}$ converges as $t\to \infty$.

    We begin by rephrasing the Gershgorin circle theorem in our language. 
    \begin{lemma}[\textbf{Gershgorin circle theorem}]\label{lemma-Gershgorin}
        Consider a square matrix $G$ where the element in the $i$-{th} row and $j$-{th} column is $G_j^i$. Denote $R_j$ the sum of the absolute values of the non-diagonal entries in the $j$-{th} column
        \begin{align}
            R_j=\sum_{i:i\neq j} |G_j^i|.
        \end{align}
        Let $D(G_j^j, R_j)\subset \sC$ be a closed disc centered at $G_j^j$ with radius $R_j$ in the complex plane $\sC$.
        
        Then, each eigenvalue of $G$ must lie in one of such discs. 
    \end{lemma}
    Suppose our configuration set $\gset$ has size $m$. Then, we can assign each $\field\in \gset$ a unique integer in $[m]:=\{1, 2, \dots, m\}$. Thus, we find a matrix representation $G$ of our generator $\bG$. We may as well use $G_{\field}^{\field'}$ to denote the $\field$-{th} row and $\field'$-{th} column of the matrix. 

    By Definition~\ref{def:inf-generator}, we know $G^{\field}_{\field}\leq 0$, $G^{\field'}_{\field}\geq 0$ if $\field'\neq \field$,  and $\sum_{\field'}G^{\field'}_{\field}=0$. Therefore,
    \begin{align}
        R_{\field}=\sum_{\field':\field'\neq \field} |G_{\field}^{\field'}|=\sum_{\field':\field'\neq \field} G_{\field}^{\field'}=-G_{\field}^{\field}=|G_{\field}^{\field}|.
    \end{align}
    As a result, the disc $D(G_{\field}^{\field}, R_j)$ must lie in the non-positive half of the complex plane, and the only point in the disc with non-negative real part is exactly $0$. 

    Next, we turn our attention to $e^{\bG t}$.
    Let $J$ be the Jordan normal form of matrix $G$, and there exists invertible matrix $P$ such that 
    \begin{align}
        G=PJP^{-1}.
    \end{align}
    Denote the Jordan blocks by $J_k$ with eigenvalue $\lambda_k$. Note that 
    \begin{align}
        e^{Gt}&=e^{PJP^{-1}t}=1+\sum_{n=1}^\infty \frac{(PJP^{-1}t)^n}{n!}\\
        &=P\left( 1+\sum_{n=1}^\infty \frac{J^nt^n}{n!} \right) P^{-1}\\
        &=P e^{Jt} P^{-1}
    \end{align}
    Thus, the convergences of the above requires the convergence of every Jordan blocks 
    \begin{align}
        e^{J_k t}=1+\sum_{n=1}^\infty \frac{J_k^nt^n}{n!} 
    \end{align}
    Note that each Jordan block can be decomposed by 
    \begin{align}
        J_k = \lambda_k I_k + N_k,
    \end{align}
    where $I_k$ is the $k\times k$ identity matrix, and $N_k$ is a nilpotent matrix satisfying $N_k^n=0$ for all $n\geq k$.
    Since $I_k, N_k$ commute, we have
    \begin{align}
        e^{J_k t}=e^{\lambda_k t}e^{N_k t}.
    \end{align}
    Observe 
    \begin{align}
        e^{N_k t}=1+\sum_{n=1}^{k-1} \frac{N_k^nt^n}{n!} = O(t^{k-1}).
    \end{align}
    Thus, if $\lambda_k$ has negative real parts, we have
    \begin{align}
        \lim_{t\to \infty} e^{\lambda_k t}e^{N_k t} \to 0.
    \end{align}
    Otherwise, $\lambda_k=0$, and  
    \begin{align}
        e^{J_k t}=e^{N_k t}=1+\sum_{n=1}^{k-1} \frac{N_k^nt^n}{n!}. \label{eq:jordan-1}
    \end{align}  

    By Proposition~\ref{prop:dynamics-generator}, we know that 
    \begin{align}
        \ket{\varphi(t)}=e^{\bG t}\ket{\field}
    \end{align}
    represents a probability distribution given any basis vector  $\ket{\field}$. Therefore, $e^{\bG t}$ is always bounded by a constant. 

    Observe from \eqref{eq:jordan-1} that if $\lambda_k=0$ then it must be $k\leq 1$, otherwise \eqref{eq:jordan-1} is unbounded. Therefore, in this case $e^{J_k t}=1$ and thus the convergence automatically holds.

    Hence, we deduce that $e^{Jt}$ converges as $t\to \infty$, and thus the convergence of $e^{\bG t}=Pe^{Jt}P^{-1}$. This implies the existence of 
    \begin{align}
        \ket{\varphi(\infty)}:=\lim_{t\to \infty} \ket{\varphi(t)}=\lim_{t\to \infty} e^{\bG t} \ket{\varphi(0)} .
    \end{align}

\end{proof}

\subsection{Corollary~\ref{cor:avg-valid}}\label{proof-cor:avg-valid}
\textbf{Corollary~\ref{cor:avg-valid} ($\avgfield$ and $\ppsv$ Always Exist).}
    Given that the configuration set $\gset$ is finite, the averaged state $\avgfield$ and the objective value $\ppsv$ always exist. Moreover, we have 
    \begin{align}
        \avgfield=\ket{\varphi(\infty)}\qquad \text{and}\qquad \ppsv=\pps \lG \avgfield.
    \end{align}

\begin{proof}
    Proposition~\ref{prop:stationary} states that $\ket{\varphi(t)}$ is converging in a Hilbert space $\gH$. Thus, its mean would converge to the same limit, i.e., $\avgfield=\ket{\varphi(\infty)}$, and thus always exists.

    Moreover, since $\pps$ and $\lG$ are both continuous linear, we have 
    \begin{align}
        \ppsv := \lim_{T\to \infty} \frac{1}{T} \int_{0}^T \dt\ \pps \lG \ket{\varphi(t)}= \pps \lG \lim_{T\to \infty} \frac{1}{T} \int_{0}^T \dt\  \ket{\varphi(t)} = \pps \lG\avgfield,
    \end{align}
    and thus $\ppsv$ always exist.

\end{proof}

\subsection{Theorem~\ref{thm:ID-case}}\label{proof-thm:ID-case}

\textbf{Theorem~\ref{thm:ID-case} (Existence of \dnameU: A Case Study).}
    Consider the two-entity formulation described above for an ergodic system. If the objective operator $\pps:\lH \to \sR_+$ is non-negative and $\forall{\bN\in \genset_{x'}}:\min_{\bM\in \genset_x} \ppsv(\bM)=0 $ for a finite $\genset_x$, then there exists generator $\bM^\star$  that simulates the minimization, such that
    \begin{align}
        \forall \bN\in \genset_{x'}:\quad \ppsv(\bM^\star)= \min_{\bM\in \genset_x}\  \ppsv(\bM)=0. 
    \end{align}
\begin{proof}

    The idea is to simply construct such an $\bM^\star$, where it tries a different strategy $\bM$ if the current one does not work. The construction is as follows.

    Denote the size of $\genset_x$ to be $m$, i.e., there are $m$ generators to choose from, where each $\bM:\gA\times \gB\times \gM \to \gA\times \gM$. Let us label these $m$ generators by $M_i$ for $i\in [m]$.

    To simulate different $\bM$, we need an indicator to keep track of which one is currently being simulated. Concretely, suppose entity $x$ is simulating the $i$-th strategy with $(\alpha,\beta,\mu)\in \gA\times \gB\times \gM$ being the current configuration of the simulated one, then we need $(\alpha, \beta, (\mu, i))$ to be the new configuration for $\bM^\star$. 
    
    Therefore, we construct $\gM^\star$ as 
    \begin{align}
        \gM^\star:= \gM\times [m].
    \end{align} 
    However, when $\bM^\star$ is simulating a $\bM$, we may still work with the corresponding Hilbert space $\lH$.
    Then, we can construct the new generator $\bM^\star$ as the following. For each $\field\to \field^+$, if its objective signal is $\pps \ket{\field^+\field}= 0$, then $\bM^\star$ keeps simulating the current $\bM$. If its objective signal is positive, then $\bM^\star$ has a positive transition rate to other strategies in $\genset_x$ by setting the current strategy $i$ to a uniformly random label in $j\in [m]$. 

    Next, we prove that this $\bM^\star$ leads to the convergence to a stationary distribution with zero objective value. We know that $\forall \bN\in \genset_{x'}$, there are some $\bM_i\in \genset_x$  such that $\ppsv(\bM_i)=0$. Let $\avgfield_i$ denote the corresponding stationary distribution, and by definition $\ppsv(\bM_i)=\pps \lG \avgfield_i=0$.

    Let $\gset_i\subset \gset$ be the subset of configurations where $\avgfield_i$ has positive measure, i.e., $\avgfcoef_i^\field > 0$ if and only if $\field\in \gset_i$. We can see that, by definition, $\gset_i$ must be closed, meaning that any configuration in $\gset_i$ can only transit to another configuration in $\gset_i$. Moreover, this transition can never incur positive objective signal $\pps$, since $\pps \lG \avgfield_i=0$. The important property we can see is that $\gset_i$ is absorbing, i.e., once the system configuration transits into this set, it never leaves, and it never incurs positive objective value.  
    
    Take any configuration $(\alpha,\beta,\mu,\nu,j)\in \gset\backslash \gset_i$. If the $j$-th strategy has objective value $\ppsv(M_j)>0$, then it must have a path with positive transition rate to change to strategy $i$ by construction of the $\bM^\star$. If  $j=i$ but still $(\alpha,\beta,\mu,\nu,i)\notin\gset_i$, then it must have a positive path leading to $\gset_i$ because the simulated process $M_i$ would converge to that. 
    
    Therefore, under $\bM^\star$, for any configuration of the system outside such a $\gset_i$, it must have a path with positive transition rate leading to one of such $\gset_i$. Since such $\gset_i$ are all absorbing, the system eventually must stay in one of those forever, and thus the corresponding stationary state $\avgfield_i$ incurs zero objective value. 

\end{proof}

\subsection{Fact~\ref{fact:subH}}\label{proof-fact:subH}
\textbf{Fact~\ref{fact:subH} (Another Representation of $\subH$).}
Write $\sfield=\sum_{\field}\sfcoef^\field \ket{\field}$, and we have
    \begin{align}
        \subH=\proj\ \lH=\left\{\sfield \in \gH\mid \sum_{\field}\sfcoef^\field =0 \right\}.
    \end{align}
\begin{proof}
    Let us denote $(\subH)':=\left\{\sfield \in \gH\mid \sum_{\field}\sfcoef^\field =0 \right\}$. By definition of the projection operator $\proj$ (\eqref{eq:proj}), it is easy to see that $\proj\lH\subseteq (\subH)'$. 

    For the other direction, suppose $\dim(\gH)=n$, and we can see that $(\subH)'$ is a $(n-1)$-dimensional subspace of $\gH$. Let us fix $\field$ and choose the following basis vector for consideration:
    \begin{align}
        \{\ket{\field'}-\ket{\field}\}_{\field'\in \gset}.
    \end{align}
    The above is a set of $n-1$ linearly independent vectors (ignoring the zero), and they all lie in the subspace $(\subH)'$ whose dimension is also $n-1$. Therefore,  the above set of vectors span the entire $(\subH)'$, i.e., 
    \begin{align}
        (\subH)'=\texttt{span}\left(\{\ket{\field'}-\ket{\field}\}_{\field'\in \gset}\right).
    \end{align}
    Finally, observe  
    \begin{align}
        (\subH)'=\texttt{span}\left(\{\ket{\field'}-\ket{\field}\}_{\field'\in \gset}\right)=\proj\ \texttt{span}\left(\{\ket{\field'\field}\}_{\field'\in \gset}\right)\subseteq \proj\ \lH.
    \end{align}
    Therefore, the above proves $(\subH)'=\proj\lH=\subH$.
\end{proof}

\subsection{Fact~\ref{fact:proj-G}}\label{proof-fact:proj-G}
\textbf{Fact~\ref{fact:proj-G}.} By definition 
	\begin{align}
		\bG=\proj\ \lG.
	\end{align}   
\begin{proof}
    By definition,
    \begin{align}
        \bG=\sum_{\field,\field'\in \gset}\ket{\field'}G^{\field'}_{\field}\bra{\field}, \qquad \lG=\sum_{\field,\field'\in \gset}G^{\field'}_{\field} \act_{\field'}^{\field}.
    \end{align}
    Therefore, for $\forall \field\in \gset$, 
    \begin{align}
        \proj \ \lG\ket{\field}&=\sum_{\field'}\proj \ G^{\field'}_{\field} \ket{\field'\field}=\sum_{\field'} G^{\field'}_{\field}\cdot (\ket{\field'}-\ket{\field})\\
        &=\sum_{\field':\field'\neq \field} G^{\field'}_{\field}\cdot (\ket{\field'}-\ket{\field}). \label{eq:proj-G-1}
    \end{align}
    Note that by definition of the generator, $\sum_{\field':\field'\neq \field} G^{\field'}_{\field}=-G^\field_\field$. Therefore, 
    \begin{align}
        (\ref{eq:proj-G-1})=\sum_{\field'}G^{\field'}_{\field}\ket{\field'}=\bG\ket{\field}.
    \end{align}
\end{proof}

\subsection{Proposition~\ref{prop:bound-S}}\label{proof-prop:bound-S}
\textbf{Proposition~\ref{prop:bound-S}.}
    The operator $\bS:\subH\to \gH$ is bounded if the system is ergodic. 

\begin{proof}
    Given that the system is ergodic, there is a unique stationary state $\avgfield$. 
    In this case, it is known that the dynamics converge to the stationary state with an exponential rate. Formally, for any initial state $\ket{\varphi(0)}$, there exist some constant $K, \kappa>0$ such that
    \begin{align}
        \left\|e^{\bG t}\ket{\varphi(0)}-\avgfield\right\|\leq Ke^{-\kappa t}.
    \end{align} 
    A proof of this, for example, can be found at \citep[Lemma~4.4]{yin2012continuous}.

    Next, note that $\subH$ is the span of vectors like $\ket{\field'}-\ket{\field}$, and we can see that any $\Delta \sfield\in \subH$ can be expressed by the difference of two states (normalized vectors), i.e.,
    \begin{align}
        \Delta \sfield = Z\cdot (\sfield_1-\sfield_2),
    \end{align}
    where $Z>0$ and $\sfield_1, \sfield_2\in \gH$ are normalized states representing probability distribution.

    Therefore, 
    \begin{align}
        \left\| e^{\bG t}\Delta \sfield \right\|&=Z\cdot \left\| e^{\bG t}(\sfield_1-\sfield_2) \right\|=Z\cdot\left\| e^{\bG t}\sfield_1-\avgfield+\avgfield-e^{\bG t}\sfield_2 \right\|\\
        &\leq Z\cdot\left\| e^{\bG t}\sfield_1-\avgfield\right\|+Z\cdot\left\|\avgfield-e^{\bG t}\sfield_2 \right\|\\
        &\leq 2ZKe^{-\kappa t}.
    \end{align}

    Plugging this into the definition of operator $S$, we obtain
    \begin{align}
        \left\|\bS \Delta \sfield \right\|&= \left\| \int_{0}^{\infty}\dt \ e^{\bG t}\Delta \sfield\right\|\leq \int_{0}^{\infty}\dt \ \left\|  e^{\bG t}\Delta \sfield\right\|\\
        &\leq \int_{0}^{\infty}\dt\ 2ZKe^{-\kappa t} <\infty.
    \end{align}

\end{proof}

\subsection{Proposition~\ref{prop:grad}}\label{proof-prop:grad}
\textbf{Proposition~\ref{prop:grad} (Gradient Formula).}
    Given an objective operator $\pps:\lH\to \sR$, assuming the system is ergodic, with the parameterization given by Definition~\ref{def:reparam-generator} we have 
    \begin{align}
        \frac{\partial \ppsv}{\partial G^{\field'}_{\field}}=\pps (1+\lG \bS \proj)\act_{\field'}^{\field} \avgfield.
    \end{align}

\begin{proof}
    Consider we apply an infinitesimal variation $\delta \bG=\epsilon \bH$ to the original generator $\bG$ denoted as
    \begin{align}
        \delta \bG=\epsilon\cdot \sum_{\field\field'}\proj\ \act_{\field'}^{\field} H^{\field'}_{\field}.
    \end{align}
    It would incur a variation $\delta \ppsv$ in the objective value. Thus, we only need to investigate how $\delta \bG$ incurs $\delta \ppsv$. 

    To do this, we need a technical lemma as stated below, which appeared in various fields~\citep{NAJFELD1995321, feynman1951operator, bellman1997introduction}.

    \begin{lemma}\label{lemma:directional-derivative}
        Consider we apply an infinitesimal variation $\delta \bG=\epsilon \bH$ to the generator $\bG$, and define the directional derivative is
        \begin{align}
            D_H(t, \bG) := \lim_{\epsilon\to 0}\frac{e^{t(\bG+\epsilon \bH)}-e^{t\bG}}{\epsilon}.        
        \end{align}
        We have
        \begin{align}
            D_H(t, \bG) = \int_0^t \ds\  e^{s\bG} \bH e^{(t-s)\bG}.
        \end{align}

    \end{lemma}
    \begin{proof}
        Consider the following differential equation
        \begin{align}
            \frac{\diff }{\dt} \sfield=(\bG+\epsilon \bH)\sfield,
        \end{align}
        and as we know its solution is 
        \begin{align}
            \ket{\varphi(t)}=e^{t(\bG+\epsilon \bH)}\ket{\varphi(0)}.
        \end{align}
        Moreover, it is easy to verify that the following is also a solution, i.e.,
        \begin{align}
            \ket{\varphi(t)}=e^{t\bG}\ket{\varphi(0)}+\epsilon\int_0^t \ds\  e^{(t-s)\bG} \bH \ket{\varphi(s)}.
        \end{align}
        Solving this iteratively, we can obtain
        \begin{align}
            \ket{\varphi(t)}=\left(e^{t\bG}+\epsilon\int_0^t \ds\  e^{(t-s)\bG} \bH e^{s\bG}\right)\ket{\varphi(0)} + \gO(\epsilon^2).
        \end{align}
        Since first order differential equation has a unique solution and that the initial point $\ket{\varphi(0)}$ is arbitrary, it must be that 
        \begin{align}
            e^{t(\bG+\epsilon \bH)}=e^{t\bG}+\epsilon\int_0^t \ds\  e^{(t-s)\bG} \bH e^{sG} + \gO(\epsilon^2).
        \end{align}
        Therefore,
        \begin{align}
            D_H(t, \bG) &:= \lim_{\epsilon\to 0}\frac{e^{t(\bG+\epsilon \bH)}-e^{tG}}{\epsilon}\\
            &=\lim_{\epsilon\to 0}\frac{\epsilon\int_0^t \ds\  e^{(t-s)\bG} \bH e^{sG} + \gO(\epsilon^2)}{\epsilon}\\
            &=\int_0^t \ds\  e^{(t-s)\bG} \bH e^{sG} \\
            &=\int_0^t \ds\  e^{sG} \bH e^{(t-s)\bG} .
        \end{align}
    \end{proof}

    We want to make sure that $\delta \avgfield =\gO(\epsilon)$ given the perturbation $\epsilon \gH$ on the generator $\bG$. This is indeed the case as we can see from the following argument. Given the ergodicity, $\avgfield$ is the unique solution to $(1+\bG)\sfield=\sfield$, i.e., $\avgfield$ is the eigenvector correspond to the non-degenerate eigenvalue $1$ of the linear operator $(1+\bG)$. It is known that analytic perturbation of a linear operator leads to analytic perturbation of the eigenvectors corresponding to non-degenerate eigenvalues~\citep{kato2013perturbation}. Therefore, the perturbation of $\delta \bG=\epsilon \bH$ leads to analytic perturbation of $\delta \avgfield=\gO(\epsilon)$. 
    
    Next, we begin from 
    \begin{align}
        \ppsv = \pps \lG\avgfield.
    \end{align}
    Applying the variation $\delta \bG$, we have 
    \begin{align}
        \delta \ppsv=\pps (\delta\lG)\avgfield+\pps \lG(\delta \avgfield) + o(\epsilon),
    \end{align}
    where $\delta\lG=\epsilon\cdot  \sum_{\field\field'}\act_{\field'}^{\field} H^{\field'}_{\field}=\gO(\epsilon)$ and $\delta\avgfield=\gO(\epsilon)$. 
    Then, by Lemma~\ref{lemma:directional-derivative},
    \begin{align}
        \delta \avgfield&= \lim_{t\to \infty} e^{(\bG+\delta \bG)t} \ket{\varphi(0)}- e^{\bG t} \ket{\varphi(0)}=\lim_{t\to \infty} \left(e^{(\bG+\delta \bG)t} - e^{\bG t} \right)\ket{\varphi(0)}\\
        &=\lim_{t\to \infty} \epsilon\cdot \int_0^t \ds\  e^{s\bG} \bH e^{(t-s)\bG} \ket{\varphi(0)} + o(\epsilon).
    \end{align}
    Given the ergodicity, the above holds for any initial state $\ket{\varphi(0)}$. Thus, let us take $\ket{\varphi(0)}=\avgfield$. Noticing that $\avgfield$ is a fixed-point with respect to the system dynamics, i.e., $ e^{(t-s)\bG} \avgfield=\avgfield$, we have
    \begin{align}
        \delta \avgfield&=\lim_{t\to \infty} \epsilon\cdot \int_0^t \ds\  e^{s\bG} \bH e^{(t-s)\bG} \avgfield + o(\epsilon)\\
        &=\epsilon\cdot \int_0^\infty \ds\  e^{s\bG}\bH \avgfield + o(\epsilon)\\
        &=\epsilon \cdot \bS\bH \avgfield+o(\epsilon)\\
        &=\bS \delta \bG \avgfield + o(\epsilon).
    \end{align}
    One can verify that the above is valid as $\bH \avgfield= \proj \sum_{\field\field'}\act_{\field'}^{\field} H^{\field'}_{\field}\avgfield\in \proj \lH$ indeed lies in $\subH=\proj \lH$, and thus $\bS\bH $ is bounded.

    Putting the above together, and we finally have
    \begin{align}
        \delta \ppsv &=\pps \delta \bG\avgfield + \pps\lG \bS \delta \bG \avgfield  +o(\epsilon)\label{eq:prop-grad-0}\\
        &=\epsilon\cdot\pps \sum_{\field\field'}\act_{\field'}^{\field} H^{\field'}_{\field}\avgfield +\epsilon\cdot \pps\lG \bS \proj \sum_{\field\field'}\act_{\field'}^{\field} H^{\field'}_{\field}\avgfield  +o(\epsilon)\\
        &=\epsilon\cdot\pps (1+\lG \bS\proj)\sum_{\field\field'}\act_{\field'}^{\field} H^{\field'}_{\field}\avgfield  +o(\epsilon).
    \end{align}
    Therefore, the above shows  
    \begin{align}
        \frac{\partial \ppsv}{\partial G^{\field'}_{\field}}=\pps(1+\lG \bS\proj)\act_{\field'}^{\field} \avgfield.
    \end{align}
\end{proof}

\subsection{Fact~\ref{fact:lifted-decompose}}\label{proof-fact:lifted-decompose}
\textbf{Fact~\ref{fact:lifted-decompose}.}
    With the above definition, we have
    \begin{align}
        \bM=\proj\ \lM, \qquad \bN=\proj\ \lN.
    \end{align}
\begin{proof}
    We can verify $\bM=\proj\lM$, and then $\bN=\proj\lN$ would hold by symmetry. By definition,
    \begin{align}
        \proj\ \lM \ \ket{\alpha\beta\mu\nu}&= \proj\ M^{\alpha'\mu'}_{\alpha\beta\mu}\ket{{\alpha'\beta\mu'\nu},{\alpha\beta\mu\nu}}\\
        &=\sum_{\alpha'\mu'}M^{\alpha'\mu'}_{\alpha\beta\mu}(\ket{{\alpha'\beta\mu'\nu}}- \ket{\alpha\beta\mu\nu})\\
        &=\sum_{\alpha'\mu':\alpha'\mu'\neq \alpha\mu}M^{\alpha'\mu'}_{\alpha\beta\mu}(\ket{{\alpha'\beta\mu'\nu}}- \ket{\alpha\beta\mu\nu})\\
        &=\sum_{\alpha'\mu'} M^{\alpha'\mu'}_{\alpha\beta\mu}\ket{{\alpha'\beta\mu'\nu}} \tag{because $\sum_{\alpha'\mu'}M^{\alpha'\mu'}_{\alpha\beta\mu}=-M^{\alpha\mu}_{\alpha\beta\mu}$}\\
        &=\bM \ket{\alpha\beta\mu\nu}.
    \end{align}

\end{proof}

\subsection{Proposition~\ref{prop:grad-entity}}\label{proof-prop:grad-entity}

\textbf{Proposition~\ref{prop:grad-entity} (Gradient Formula in the Two-entity View).}
    Given an objective operator $\pps:\lH\to \sR$, assuming the system is ergodic, with the parameterization of generator $\bM$ given by Definition~\ref{def:lift-generator-entity} we have 
    \begin{align}
        \frac{\partial \ppsv}{\partial M^{\alpha'\mu'}_{\alpha\beta\mu}}=\pps (1+\lG \bS\proj)\act_{\alpha'\mu'}^{\alpha\beta\mu} \avgfield.
    \end{align}

\begin{proof}
    Consider we apply an infinitesimal variation $\delta \bG=\epsilon \bH$ to the original generator $\bG$.
    Proposition~\ref{prop:grad} (\eqref{eq:prop-grad-0}) shows that it would incur a variation $\delta \ppsv=\gO(\epsilon)$ in the objective value as
    \begin{align}
        \delta \ppsv &=\delta(\pps\lG \avgfield)=\pps (\delta\lG)\avgfield+\pps \lG(\delta \avgfield) + o(\epsilon)\\
        &=\pps (\delta\lG)\avgfield+\pps \lG \bS(\delta \bG) \avgfield + o(\epsilon). \label{eq:grad-entity-1}
    \end{align}
    By Proposition~\ref{prop:decomposition} we know $\bG=\bM+\bN$, and thus when the variation is only applied to $\bM$ as $\delta \bM$ we would have $\delta \lG=\delta \lM$, and similarly $\delta \bG = \delta \bM$.

    Therefore, in this case 
    \begin{align}
        \delta \lG=\delta \lM=\sum_{\alpha'\beta'\alpha\beta\mu}\delta M^{\alpha'\mu'}_{\alpha\beta\mu}\act_{\alpha'\mu'}^{\alpha\beta\mu},\qquad \delta \bG =\delta \bM=\proj\ \delta \lM=\proj \ \sum_{\alpha'\beta'\alpha\beta\mu}\delta M^{\alpha'\mu'}_{\alpha\beta\mu}\act_{\alpha'\mu'}^{\alpha\beta\mu}.
    \end{align}
    Plugging this into \eqref{eq:grad-entity-1}, we have 
    \begin{align}
        \delta \ppsv&= \pps (\sum_{\alpha'\beta'\alpha\beta\mu}\delta M^{\alpha'\mu'}_{\alpha\beta\mu}\act_{\alpha'\mu'}^{\alpha\beta\mu})\avgfield+\pps \lG \bS\proj \ (\sum_{\alpha'\beta'\alpha\beta\mu}\delta M^{\alpha'\mu'}_{\alpha\beta\mu}\act_{\alpha'\mu'}^{\alpha\beta\mu}) \avgfield + o(\epsilon)\\
        &= \pps (1+ \lG \bS\proj)  (\sum_{\alpha'\beta'\alpha\beta\mu}\delta M^{\alpha'\mu'}_{\alpha\beta\mu}\act_{\alpha'\mu'}^{\alpha\beta\mu}) \avgfield + o(\epsilon).
    \end{align}
    This shows 
    \begin{align}
        \frac{\partial \ppsv}{\partial M^{\alpha'\mu'}_{\alpha\beta\mu}}=\pps (1+\lG \bS\proj)\act_{\alpha'\mu'}^{\alpha\beta\mu} \avgfield.
    \end{align}
\end{proof}

\subsection{Theorem~\ref{thm:grad-field}}\label{proof-thm:grad-field}
\textbf{Theorem~\ref{thm:grad-field} (Gradient Formula in the Field).}
Given a local objective operator $\pps_x:\lH\to \sR$, assuming the system is ergodic, the gradient formula for $\ppsv(x)=\pps_x \lG \avgfield$ w.r.t. local generator $\bG(x)$ is
    \begin{align}
        \frac{\partial \ppsv(x)}{\partial G(x)_{\field}^{\field'}}=\pps_x (1+\lG \bS\proj) \act(x)^{\field}_{\field'} \avgfield.
    \end{align}
\begin{proof}
    The proof involves two steps: taking a two-entity view of the system and applying Proposition~\ref{prop:grad-entity}.

    The system configuration of a dynamical stochastic field (Definition~\ref{def:space}) is the collection of multiple local configurations. Recall from Definition~\ref{def:space} that 
    \begin{align}
		\gset=\prod_{x\in \gX} \gset_x,
	\end{align}
    where 
    \begin{align}
		\gset_x=\gM(x)\times \prod_{x': x\to x'}\gA(x, x')
	\end{align}
    is the set of configurations where entity $x$ can change. Note that $\gM(x)$ is the set of the inner configurations of entity $x$, unobservable to others, and $\gA(x, x')$ correspond to the signals which $x$ sends to another entity $x'$.   
	Moreover, the following as what is actually being seen by $x$, i.e., 
	\begin{align}
		\nbhd_x=\gM(x)\times \prod_{x': x\to x'}\gA(x, x')\times \prod_{x': x'\to x}\gA(x', x).
	\end{align}
    Therefore, we may take a two-entity view and relabeling $\gset=\gA\times \gB\times \gM\times \gN$ by denoting 
    \begin{align}
        \gA=\prod_{x': x\to x'}\gA(x, x'),\quad \gB= \prod_{x': x'\to x}\gA(x', x),\quad \gM=\gM(x), \quad \gN=\gset/\nbhd_x.
    \end{align}
    Note that $\gA\times \gM=\gset_x$, and $\gA\times \gB\times \gM=\nbhd_x$. 

    With this relabeling, it is easy to see that the generators 
    \begin{align}
        \bM=\bG(x), \qquad \bN=\sum_{x':x'\neq x}\bG(x').
    \end{align}
    Therefore, the system generator $\bG=\sum_x \bG(x)=\bM+\bN$ as desired, and thus for $\field, \field'$ agreeing on $\gset/\gset_x$ the generator parameters
    \begin{align}
        M^{\alpha'\mu'}_{\alpha\beta\mu} \quad \text{corresponds to}\quad G(x)_{\field}^{\field'}.
    \end{align}
    Similarly, with the relabeling, the action operators 
    \begin{align}
        \act(x)^{\alpha\beta\mu}_{\alpha'\mu'} \quad \text{corresponds to}\quad \act(x)^{\field}_{\field'}.
    \end{align}

    Putting the relabeled objects into Proposition~\ref{prop:grad-entity} would result in the gradient formula for the field formulation. 
    
\end{proof}

\subsection{Proposition~\ref{prop:local-ppsv}}\label{proof-prop:local-ppsv}

\textbf{Proposition~\ref{prop:local-ppsv} (Local Objective Values).}
    If the objective propagation satisfy the following conditions:
    \begin{itemize}
        \item $\ppsp_{x'x}\lG\avgfield = \lG\avgfield$ \ for every acting entity $x'$ connecting to acting entity $x$;
        \item $\pps_{x'x}\lG \avgfield =\ppsv$ \ for every environmental entity $x'$ connecting to acting entity $x$;
        \item $\gX_a$ is strongly connected;
    \end{itemize}
    then every acting entity has the same objective value $\ppsv$, i.e.,
    \begin{align}
        \forall x\in \gX_a: \ppsv(x)=\pps_x \lG \avgfield=\ppsv.
    \end{align}

\begin{proof}
    We begin by introduce a standard definition and a lemma which will become handy. For more details and proof of the lemma, one may refer to \citep{horn2012matrix}[6.2.24 to 6.2.27]
    \begin{definition}[\textbf{Irreducibly Diagonally Dominant Matrix}]\label{def:irreducible}
        A square matrix $\Theta\in \sR^{n\times n}$ is  irreducibly diagonally dominant if it satisfies the following properties.
        \begin{enumerate}
            \item $\Theta$ is irreducible, i.e., the directed graph associated to matrix $\Theta$ is strongly connected.
            \item $\Theta$ is diagonally dominant, i.e., $\forall i\in [n]: |\Theta_{i, i}|\geq \sum_{j\neq i} |\Theta_{i,j}|$.
            \item $\exists i\in[n]:|\Theta_{i, i}|> \sum_{j\neq i} |\Theta_{i,j}|$.
        \end{enumerate}
    \end{definition}
    We have the following result of an irreducibly diagonally dominant matrix.
    \begin{lemma}[\textbf{Taussky}]\label{lemma:taussky}
        An irreducibly diagonally dominant matrix is non-singular. 
    \end{lemma}

    Without loss of generality, we may view the collection of all environmental entities as a single entity by relabeling configurations sets. Then, the ``boundary conditions'' of the objective propagation equation are given by the objective signals from the environmental entity, which stays unchanged. We may label the entities by $x_0, x_1, \dots, x_n$ where $x_0$ denotes the environmental entity, and the rest $n$ entities are the acting entities. Thus, the objective propagation equation (Definition~\ref{def:ppsp-eq}) can be written as 
    \begin{align}
        \pps_{x_i}= \sum_{j=0}^n \ppsadj^{x_j}_{x_i}\pps_{x_j x_i}, \qquad \text{where}\ \ \pps_{x_j x_i}=\pps_{x_j}\ppsp_{x_j x_i},\qquad i,j=1,2,\dots, n,
    \end{align}
    and $\pps_{x_0 x_i}$ is given by the environmental entity. Multiplying both sides by $\lG\avgfield$, we then apply the first two conditions as stated in the proposition:
    \begin{align}
        \pps_{x_i}\lG\avgfield =  \sum_{j=0}^n \ppsadj^{x_j}_{x_i}\pps_{x_j x_i}\lG\avgfield =\ppsadj^{x_0}_{x_i}\ppsv + \sum_{j=1}^n \ppsadj^{x_j}_{x_i}\pps_{x_j}\lG\avgfield.
    \end{align}
   
    Denote $\Theta\in \sR^{n\times n}$ as a matrix where its elements given by $\Theta_{ij}:=\delta_{i}^j-\ppsadj_{x_i}^{x_j}$. Moreover, denote $\bm a, \bm b\in \sR^n$ as two $n$-dimensional vectors where $a_i:= \pps_{x_i}\lG \avgfield \in \sR$ and $b_i=\ppsadj^{x_0}_{x_i}\ppsv$. Then, the above equation can be formulated into 
    \begin{align}
        \Theta \bm a = \bm b. \label{eq:prop-local-ppsv-1}
    \end{align}
    We can observe that $\Theta$ is  irreducibly diagonally dominant by checking that it satisfies the properties in Definition~\ref{def:irreducible}. First, the graph associated to the acting entities is strongly connected as stated in the theorem. Second, $\Theta$ is diagonally dominant because $\forall i: \sum_{j=0}^n \ppsadj_{x_i}^{x_j}=1=\delta_i^i$ (by Definition~\ref{def:adj-weight}). Lastly, note that there are some entity $x_i$ where $\ppsadj_{x_i}^{x_0}>0$, and thus for such $x_i$ we have $\sum_{j=1}^n \ppsadj_{x_i}^{x_j}<1$. This makes $\Theta$ satisfies the third property. Therefore, $\Theta$ is  irreducibly diagonally dominant, and thus non-singular by Lemma~\ref{lemma:taussky}. 

    Therefore, since $\Theta$ is non-singular, there is only one solution to \eqref{eq:prop-local-ppsv-1}. It is easy to see, given that $\sum_{j=0}^n\ppsadj_{x_i}^{x_j}=1$, the unique solution given by is $a_i=\ppsv$ for all $i$. This means, for all entity $x$, 
    \begin{align}
     \pps_x \lG\avgfield=\ppsv.
    \end{align}
\end{proof}

\subsection{Lemma~\ref{lemma:SPiG}}\label{proof-lemma:SPiG}
\textbf{Lemma~\ref{lemma:SPiG}.}
The following equation stands. 
    \begin{align}
        1+\bS \proj \lG=1+\proj \lG \bS =\infevo.
    \end{align} 

\begin{proof}
    Recall Fact~\ref{fact:proj-G}, i.e., $\bG=\proj \lG$, and we can see that 
    \begin{align}
        \bS \proj \lG=\bS \bG = \int_{0}^{\infty}\dt \  e^{\bG t} \bG = \int_{0}^{\infty} \  e^{\bG t} \diff\left(e^{\bG t }\right)=\infevo - 1.
    \end{align}
    Similarly, since $\bS$ and $\bG$ commutes, 
    \begin{align}
        \proj \lG\bS =\bG \bS=\bS \bG = \infevo - 1.
    \end{align}
\end{proof}

\subsection{Proposition~\ref{prop:P-properties}}\label{proof-prop:P-properties}
\textbf{Proposition~\ref{prop:P-properties}. } 
The operator $\ppsp=1+\lG \bS\proj$ satisfies the following properties.
    \begin{enumerate}
        \item $\ppsp \lG \avgfield=\lG\avgfield$.
        \item $\ppsp^2=\ppsp$.
        \item Given a variation $\delta \bG$, the incurred $\delta \ppsp$ satisfies $\delta \ppsp \lG \avgfield=0$.
    \end{enumerate}

\begin{proof}
    We are going to use two facts. The first fact is $\bG=\proj \lG$, and the second is $\bG\avgfield = 0$.

    Then, we can see that
    \begin{align}
        \ppsp \lG \avgfield=(1+\lG \bS\proj)\lG \avgfield=\lG \avgfield+\lG \bS\bG \avgfield=\lG \avgfield.
    \end{align}
    The second property can be derived similarly with Lemma~\ref{lemma:SPiG}.
    \begin{align}
        \ppsp^2&=(1+\lG \bS\proj)^2=1+2\lG \bS\proj+\lG \bS\proj\lG \bS\proj\\
        &=1+2\lG \bS\proj+\lG (\infevo-1)\bS\proj\\
        &=1+\lG \bS\proj+\lG \infevo \bS\proj\\
        &=\ppsp+\lG \infevo \bS\proj.
    \end{align}
    Note that, with the ergodicity, $\infevo$ maps everything to $\avgfield$, i.e., $\infevo:\subH\to 0$. In addition, since $\infevo$ and $\bS$ are all in the form of $e^{\bG t}$, they commute. Therefore, $\infevo \bS\proj=\bS \infevo \proj=0$. This proves 
    \begin{align}
        \ppsp^2=\ppsp+\lG \infevo \bS\proj=\ppsp.
    \end{align}

    Finally,  we can prove the third property directly from the definition. 
    \begin{align}
        \delta\ppsp \lG \avgfield=\delta(\lG \bS)\proj \lG \avgfield=\delta(\lG \bS)\bG \avgfield=0.
    \end{align}

\end{proof}

\subsection{Proposition~\ref{prop:PQ-properties}}\label{proof-prop:PQ-properties}
\textbf{Proposition~\ref{prop:PQ-properties}.}
The operator $\ppsp[\ppsq]=1+\lG \bS \ppsq \proj$ satisfies the following properties.
    \begin{enumerate}
        \item $\ppsp[\ppsq] \lG \avgfield=\lG\avgfield$.
        \item $\ppsp[\ppsq]\ppsp[\ppsq']=\ppsp[(\ppsq,\ppsq')]$.
        \item Given a variation $\delta \bG$, the incurred $\delta (\ppsp[\ppsq])$ satisfies $\delta (\ppsp[\ppsq]) \lG \avgfield=0$.
    \end{enumerate}

\begin{proof}
    The proof is very similar to that of Proposition~\ref{prop:P-properties}. We are going to use two facts. The first fact is $\bG=\proj \lG$, and the second is $\bG\avgfield = 0$.

    Then, we can see that
    \begin{align}
        \ppsp[\ppsq] \lG \avgfield=(1+\lG \bS \ppsq \proj)\lG \avgfield=\lG \avgfield+\lG \bS\ppsq \bG \avgfield=\lG \avgfield.
    \end{align}
    The second property can be derived similarly with Lemma~\ref{lemma:SPiG}.
    \begin{align}
        \ppsp[\ppsq]\ppsp[\ppsq']&=(1+\lG \bS \ppsq \proj)(1+\lG \bS \ppsq' \proj)\\
        &=1+\lG \bS \ppsq \proj+\lG \bS \ppsq' \proj+\lG \bS \ppsq \proj\lG \bS \ppsq' \proj\\
        &=1+\lG \bS \ppsq \proj+\lG \bS \ppsq' \proj+\lG \bS \ppsq (\infevo-1) \ppsq' \proj\\
        &=1+\lG \bS (\ppsq,\ppsq') \proj+\lG \bS \ppsq \infevo \ppsq' \proj\\
        &=1+\lG \bS (\ppsq,\ppsq') \proj,
    \end{align}
    where the last step is because that the image of $\ppsq'$ is in $\subH$ and $\infevo:\subH\to 0$.
    This proves 
    \begin{align}
        \ppsp[\ppsq]\ppsp[\ppsq']=\ppsp[(\ppsq,\ppsq')].
    \end{align}

    Finally, we can prove the third property directly from the definition. 
    \begin{align}
        \delta(\ppsp[\ppsq]) \lG \avgfield= \delta(\lG \bS \ppsq)\proj \lG \avgfield=\delta(\lG \bS \ppsq)\bG \avgfield=0.
    \end{align}
\end{proof}

\subsection{Theorem~\ref{thm:ppsq-grad}}\label{proof-thm:ppsq-grad}
\textbf{Theorem~\ref{thm:ppsq-grad} ($\ppsp[\ppsq]$  Allows for Local Gradient Computations). }
    When the acting entities $\gX_a$ are strongly connected, objective propagators of the form of $\ppsp_{x'x}=\ppsp[\ppsq_{x'x}]$ result in local gradient computations. That is,  Theorem~\ref{thm:grad-field} applies, and we obtain the same gradient formula 
    \begin{align}
        \frac{\partial \ppsv(x)}{\partial G(x)^{\field'}_{\field}}=\pps_x\ppsp\act(x)_{\field'}^{\field} \avgfield.
    \end{align}

\begin{proof}
    The proof is very similar to the proof of Proposition~\ref{prop:local-ppsv} where we need the notion of irreducibly diagonally dominant matrix (Definition~\ref{def:irreducible}) and Lemma~\ref{lemma:taussky}.

    The key ingredient in proving the theorem is to show that $\delta \pps_x \lG\avgfield=0$ for propagator $\ppsp_{x'x}=\ppsp[\ppsq_{x'x}]$. 
    We may view the collection of all environmental entities as a single entity by relabeling configurations sets. Then, the ``boundary conditions'' of the objective propagation equation are given by the objective signals from the environmental entity, which stays unchanged. We may label the entities by $x_0, x_1, \dots, x_n$ where $x_0$ denotes the environmental entity, and the rest $n$ entities are the acting entities. Thus, the objective propagation equation (Definition~\ref{def:ppsp-eq}) can be written as 
    \begin{align}
        \pps_{x_i}= \sum_{j=0}^n \ppsadj^{x_j}_{x_i}\pps_{x_j x_i}, \qquad \text{where}\ \ \pps_{x_j x_i}=\pps_{x_j}\ppsp[\ppsq_{x_j x_i}],\qquad i,j=1,2,\dots, n,
    \end{align}
    and $\pps_{x_0 x_i}$ is given by the environmental entity.
    Next, taking the variation and noting that $\pps_{x_0x_i}$ is unchanged, we obtain 
    \begin{align}
        \delta \pps_{x_i} = \sum_{j=1}^n  \ppsadj_{x_i}^{x_j}\left(  \pps_{x_j} \delta (\ppsp[\ppsq_{x_j x_i}])+\delta \pps_{x_j} \ppsp[\ppsq_{x_j x_i}] \right).
    \end{align}
    Then, The above equation can be very much simplified by multiplying $\lG \avgfield$ on both sides and applying Proposition~\ref{prop:PQ-properties}.
    \begin{align}
        \delta \pps_{x_i} \lG \avgfield &= \sum_{j=1}^n  \ppsadj_{x_i}^{x_j}\left( \pps_{x_j} \delta (\ppsp[\ppsq_{x_j x_i}]) \lG \avgfield+\delta \pps_{x_j} \ppsp[\ppsq_{x_j x_i}]\lG \avgfield\right)\\
        &=\sum_{j=1}^n  \ppsadj_{x_i}^{x_j}\delta \pps_{x_j} \lG \avgfield.
    \end{align}
    Denote $\Theta\in \sR^{n\times n}$ as a matrix where its elements given by $\Theta_{ij}:=\delta_{i}^j-\ppsadj_{x_i}^{x_j}$. Moreover, denote $\bm a\in \sR^n$ as a vector where   $a_i:=\delta \pps_{x_i}\lG \avgfield \in \sR$. Then, the above equation can be formulated into 
    \begin{align}
        \Theta \bm a = 0. \label{eq:ppsq-grad-1}
    \end{align}
    We can observe that $\Theta$ is  irreducibly diagonally dominant by checking that it satisfies the properties in Definition~\ref{def:irreducible}. First, the graph associated to the acting entities is strongly connected as stated in the theorem. Second, $\Theta$ is diagonally dominant because $\forall i: \sum_{j=0}^n \ppsadj_{x_i}^{x_j}=1=\delta_i^i$ (by Definition~\ref{def:adj-weight}). Lastly, note that there are some entity $x_i$ where $\ppsadj_{x_i}^{x_0}>0$, and thus for such $x_i$ we have $\sum_{j=1}^n \ppsadj_{x_i}^{x_j}<1$. This makes $\Theta$ satisfies the third property. Therefore, $\Theta$ is  irreducibly diagonally dominant, and thus non-singular by Lemma~\ref{lemma:taussky}. 

    Therefore, since $\Theta$ is non-singular, the only solution to \eqref{eq:ppsq-grad-1} is $\bm a=0$. This means, for all entity $x$, 
    \begin{align}
        \delta \pps_x \lG\avgfield=0.
    \end{align}

    As a result, the variation on the local objective value enjoys a simple formula as 
    \begin{align}
        \delta \ppsv = \delta(\pps_x \lG\avgfield)= \pps_x \delta(\lG\avgfield) + \delta\pps_x \lG\avgfield = \pps_x \delta(\lG\avgfield).
    \end{align}    
    It means that entity $x$ may view $\pps_x$ as a fixed objective operator when taking the gradient, and it reduces the setting to where Theorem~\ref{thm:grad-field} applies.

\end{proof}